\documentclass[conference]{IEEEtran}
\IEEEoverridecommandlockouts
\usepackage{cite}
\usepackage{amsmath,amssymb,amsfonts}
\usepackage{algorithmic}
\usepackage{graphicx}
\usepackage{textcomp}

\usepackage{comment}
\usepackage{algorithmic,algorithm}
\usepackage{pifont}

\usepackage{subcaption}
\usepackage{booktabs}
\usepackage{epsfig}
\usepackage{epstopdf}

\usepackage{amsthm}

\newtheorem{theorem}{Theorem}
\newtheorem{assumption}{Assumption}

\newtheorem{remark}{Remark}

\newtheorem{lemma}{Lemma}

\newcommand{\cmark}{\ding{51}}%
\newcommand{\xmark}{\ding{55}}%

\usepackage{xcolor}
\def\BibTeX{{\rm B\kern-.05em{\sc i\kern-.025em b}\kern-.08em
    T\kern-.1667em\lower.7ex\hbox{E}\kern-.125emX}}
\begin{document}

\title{Straggler-Agnostic and Communication-Efficient Distributed Primal-Dual Algorithm for High-Dimensional Data Mining}


\author{\IEEEauthorblockN{Zhouyuan Huo, Heng Huang}
\IEEEauthorblockA{\textit{Department of Electrical and Computer Engineering} \\
\textit{University of Pittsburgh}\\
Pittsburgh, United States \\
zhouyuan.huo, heng.huang@pitt.edu}
}

\maketitle
\begin{abstract}
	Recently, reducing the communication time between machines becomes the main focus of the distributed data mining. Previous methods propose to make workers do more computation locally before aggregating local solutions in the server such that fewer communication rounds between server and workers are required. However, these methods do not consider reducing the communication time per round and work very poor under certain conditions, for example, when there are straggler problems or the dataset is of high dimension. In this paper, we target to reduce communication time per round as well as the required communication rounds. We propose a communication-efficient distributed primal-dual method with straggler-agnostic server and bandwidth-efficient workers. We analyze the convergence property and prove that the proposed method guarantees linear convergence rate to the optimal solution for convex problems. Finally,  we conduct large-scale experiments in simulated and real distributed systems and experimental results demonstrate that the proposed method is much faster than compared methods.
\end{abstract}


\section{Introduction}
Distributed optimization methods are nontrivial when we optimize a data mining problem when the data or model is distributed across multiple machines. When data are distributed, parameter server \cite{dean2012large,li2014communication} or decentralized methods \cite{lian2017can,lian2017asynchronous} were proposed for parallel computation and linear speedup. When model are distributed, especially deep learning model, pipeline-based methods or decoupled backpropagation algorithm \cite{huo2018decoupled,xu2019diversely,huo2018training,yang2019ouroboros}  parallelized the model updating on different machines and made full use of computing resources. In this paper, we only consider the case that data are distributed.

As in Figure \ref{sh}, most distributed methods require collecting update information from all workers iteratively to find the optimal solution \cite{boyd2011distributed,jaggi2014communication,lee2015distributed,huo2017asynchronous,huo2018asynchronous}. As the communication in the network is slow, it is challenging to obtain accurate solutions if there is a limited time budget. Total running time of the distributed methods is determined by multiple factors. To get  $\varepsilon$-accurate solutions, the total running time of a distributed algorithm can be represented as follows:
\begin{eqnarray}
\mathcal{T}(\mathcal{A}, \varepsilon) = \sum\limits_{t=1}^{I(\mathcal{A},\varepsilon)} \left( \mathcal{T}_c(d) + \max_k \mathcal{T}_{\mathcal{A},t}^k  \right), 
\end{eqnarray}
where $I(\mathcal{A},\varepsilon)$ denotes the number of  communication rounds the algorithm $\mathcal{A}$ requires to get $\varepsilon$-accurate solutions and $ \mathcal{T}_c(d)$ represents the communication time per round, which is dependent on the dimensionality $d$ of the data.  $ \max\limits_k \mathcal{T}_{\mathcal{A},t}^k$ indicates the computational time required by the slowest worker at round $t$, such that all workers have completed their jobs at that time.  To reduce the total running time $\mathcal{T}(\mathcal{A},\varepsilon)$, we can either decrease the number of communication rounds $I(\mathcal{A}, \varepsilon)$, or cut down the running time at each round $\left( \mathcal{T}_c(d) +\max\limits_k \mathcal{T}_{\mathcal{A},t}^k \right).$

\begin{figure}[t]
	\centering
	\includegraphics[width=2in]{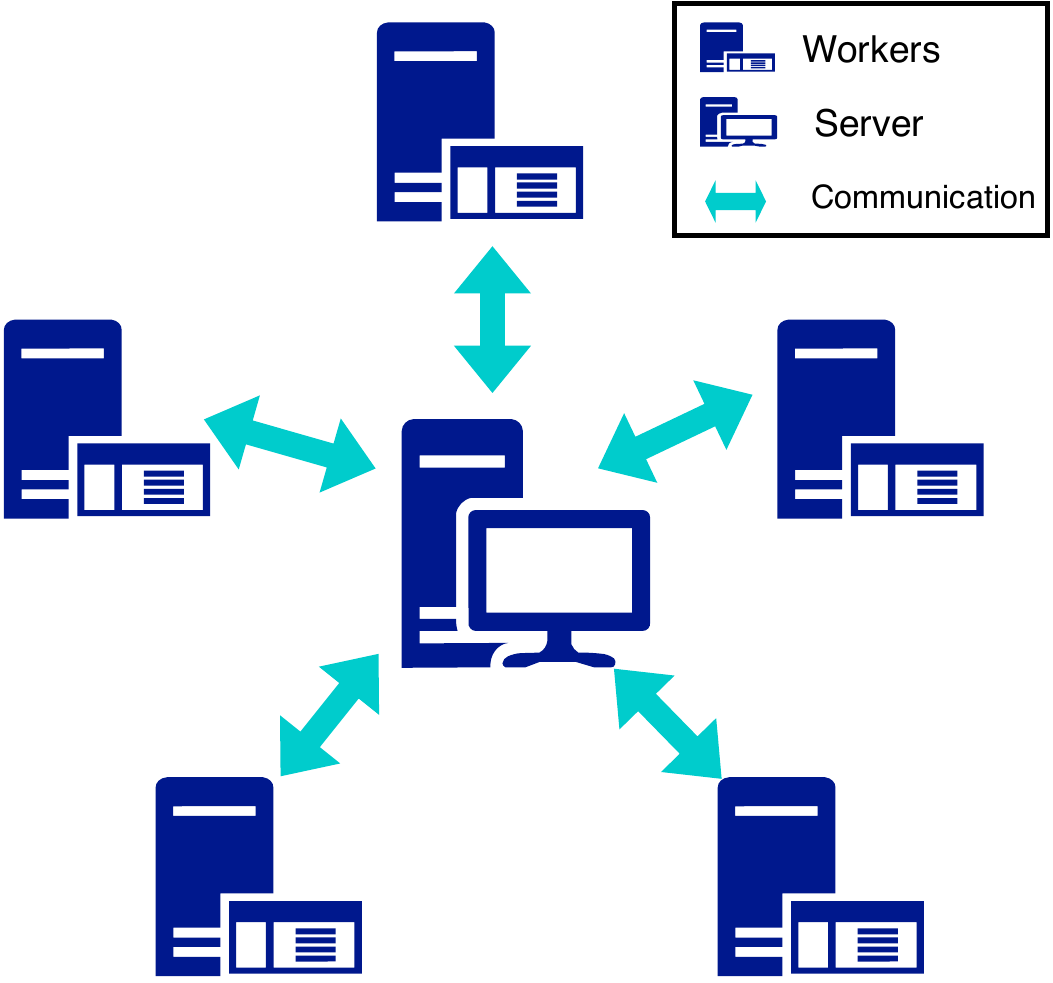}
	\caption{Parameter server structure for distributed data mining. }
	\label{sh}
\end{figure}

There are various methods trying to reduce the total running time by reducing the number of communication rounds $I(\mathcal{A}, \varepsilon)$.  Previous communication-efficient methods make workers update locally for iterations and communicate with the server periodically. For example, DSVRG \cite{lee2015distributed}, DISCO \cite{zhang2015disco}, AIDE \cite{reddi2016aide} and DANE \cite{shamir2014communication} proved that they require $O(\log \frac{1}{\varepsilon})$ communication rounds to reach $\varepsilon$-accurate solutions. There are also various distributed methods for dual problems, for example, CoCoA \cite{jaggi2014communication,smith2018cocoa}, CoCoA+ \cite{ma2015adding,smith2016cocoa} and DisDCA \cite{yang2013trading}. These methods also  admit linear convergence guarantees regarding communication rounds. 
In \cite{ma2015adding}, the authors proved that CoCoA+ is a generation for CoCoA  and showed that CoCoA+ is  equivalent to DisDCA under certain conditions. A brief comparison of distributed primal-dual methods is in Table \ref{table:methods}. These communication-efficient  methods work well when the communication time per round is relatively small and all workers run with similar speed.  
However, they suffer from the communication bottleneck $\mathcal{T}_c(d)$ when the data is of high dimensionality or straggler problem where there are machines work far slower than other normal workers. 

In  \cite{aji2017sparse,alistarh2017qsgd,strom2015scalable,lin2017deep}, authors proposed to reduce the communication time $\mathcal{T}_c(d)$ and increase the bandwidth efficiency by compressing or dropping the gradients for distributed optimization. There are also several attempts trying to quantize the gradients such that fewer messages are transmitted in the network \cite{alistarh2016qsgd,wen2017terngrad}. However, these methods are not communication-efficient. These methods ask workers to send gradient information to the server every iteration,  suffering from a large number of communication rounds. To the best of our knowledge, there is no work reducing the size of the transmitted message for distributed primal-dual methods.

In this paper, we focus on reducing the running time  at each round $ \left(\mathcal{T}_c(d) + \max\limits_k \mathcal{T}_{\mathcal{A},t}^k \right)$ for distributed data mining. To solve the issues of the straggler problem and high-dimensional data, we propose a novel straggler-agnostic and bandwidth-efficient distributed primal-dual algorithm. The main contributions of our work are summarized as follows:
\begin{itemize}
	\item   We propose a novel primal-dual algorithm to solve the straggler problem and high communication complexity per iteration in Section \ref{algo}. 
	\item We provide convergence analysis in Section \ref{ana_1} and prove that the proposed method guarantees linear convergence to the optimal solution for the convex problem. 
	\item We perform experiments with large-scale datasets distributed across multiple machines in Section \ref{exp}.  Experimental results verify that the proposed method can be up to 4 times faster than compared methods.
\end{itemize}

\begin{table}[t]
	\center
	\caption{ Communications of distributed primal-dual algorithms. $d$ denotes the size of the model, $\rho$ denotes the sparsity constant and $0<\rho \ll 1$. S-A denotes straggler agnostic.}
	\label{table:methods}
	\setlength{\tabcolsep}{3mm}
	\begin{tabular}{cccc}
		\hline 
		\textbf{Algorithm}  & \textbf{S-A} &  \textbf{$\mathcal{T}_c(d)$ }& \textbf{Communication Rounds} \\   \hline 
		DisDCA \cite{yang2013trading}& \xmark
		& $O(d)$	&  $O\left( \left(1 + \frac{1}{\lambda \mu}\right)\log(\frac{1}{\varepsilon}) \right)$   \\
		\hline
		CoCoA \cite{jaggi2014communication}& \xmark & $O(d)$ 
		&  $O \left(\left( K + \frac{1}{\lambda \mu}\right)\log(\frac{1}{\varepsilon}) \right)$ \\
		\hline
		CoCoA+ \cite{ma2015adding}& \xmark & $O(d)$  
		&  $O \left( \left(  \ 1 + \frac{1}{\lambda \mu}\right)\log(\frac{1}{\varepsilon}) \right) $\\
		\hline 
		ACPD& \cmark& $O(\rho d)$
		&  $O \left(  \left( 1 + \frac{1}{\lambda \mu} \right)\log(\frac{1}{\varepsilon}) \right) $     \\
		\hline 
	\end{tabular}
\end{table}

\section{Related Work}
\subsection{Stochastic Dual Algorithm}
In this paper, we consider to optimize the following $\ell_2$ regularized empirical loss minimization problem which is arising ubiquitously in supervised machine learning:
\begin{eqnarray}
\label{primal}
P(w) &:=& \min\limits_{w \in \mathbb{R}^d} \frac{1}{n} \sum\limits_{i=1}^n \phi_i(w^Tx_i) + \frac{\lambda}{2} \|w\|^2_2,
\end{eqnarray}
where $x_i \in \mathbb{R}^d$ denotes data sample $i$ and $w \in \mathbb{R}^d$ denotes the linear predictor to be optimized. There are many applications falling into this formulation, for example, classification problem or regression problem. 
To solve the primal problem (\ref{primal}), we can optimize its dual problem instead:
\begin{eqnarray}
D(\alpha) &:=&\max \limits_{\alpha \in \mathbb{R}^n} \frac{1}{n} \sum\limits_{i=1}^n - \phi_i^*(-\alpha_i) - \frac{\lambda}{2} \left\|\frac{1}{\lambda n} A\alpha \right\|^2_2,
\label{dual}
\end{eqnarray}
\begin{figure}[t]
	\centering
	\includegraphics[width=3.5in]{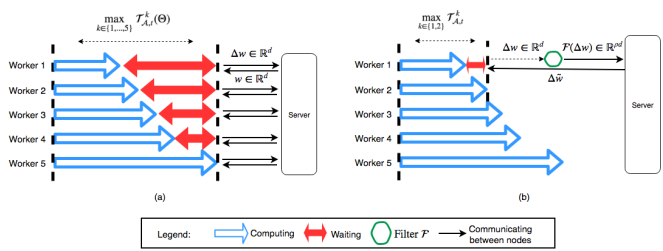}
	\caption{Communication protocols for (a) previous methods such as CoCoA, CoCoA+, and DisDCA,  (b) ACPD. \textbf{Straggler-Agnostic:} Previous methods use the synchronous communication protocol, suffering from the slowest workers in the cluster.  ACPD allows for group-wise communication between server and workers such that it avoids the straggler problem. \textbf{Bandwidth-Efficient:} Previous methods send variables $\Delta w$ with dimensionality $d$ directly through the network. ACPD inputs $\Delta w$  into a message filter and then sends compressed variable $\mathcal{F} (\Delta w)$  with dimensionality $\rho d$, where $0<\rho \ll 1$. The number of nonzero elements in $\Delta \tilde{w}$ is also of size $O(\rho d) $ on average. }
	\label{algorithm}
\end{figure}
where $\phi_i^*$ is the convex conjugate function to $\phi_i$, $A  = [x_1,x_2,...x_n] \in \mathbb{R}^{d\times n}$ denotes data matrix and $\alpha \in \mathbb{R}^n$ represents dual variables. Stochastic dual coordinate ascent (SDCA)  \cite{hsieh2008dual,shalev2013stochastic} is one of the most successful methods proposed to solve problem (\ref{primal}). In \cite{shalev2013stochastic}, the authors proved that SDCA guarantees linear convergence if the convex function $\phi_i(w^Tx_i)$ is smooth, which is much faster than stochastic gradient descent (SGD) \cite{bottou2016optimization}. At iteration $t$, given sample $i$ and  variables $\alpha_{j\neq i}$  fixed, we maximize the following subproblem:
\begin{eqnarray}
\label{dual_sub}
\max \limits_{\Delta \alpha_i \in \mathbb{R}} -\frac{1}{n}  \phi_i^*(-(\alpha_i^{t} + \Delta \alpha_i )) - \frac{\lambda}{2} \| w^{t} + \frac{1}{\lambda n}\Delta \alpha_i x_i \|^2_2.
\end{eqnarray}
$e_i$ denotes a coordinate vector of size $n$, where element $i$ is $1$ and other elements  are $0$.
Another advantage of optimizing the dual problem is that we can monitor the optimization progress by keeping track of the duality gap $G(\alpha)$. The duality gap is defined as: $G(\alpha) = P(w(\alpha)) - D(\alpha)$, where $P(w(\alpha))$ and $D(\alpha)$ denote objective values of the primal problem and the dual problem respectively. Assuming $w^*$ is the optimal solution to the primal problem (\ref{primal}), $\alpha^*$ is the optimal solution to the dual problem (\ref{dual}), the primal-dual relation is always satisfied  such that:
\begin{eqnarray}
w^*\hspace{0.2cm} = \hspace{0.2cm}  w(\alpha^*)\hspace{0.2cm}  =\hspace{0.2cm}  \frac{1}{\lambda n} A \alpha^*.
\label{primal_dual}
\end{eqnarray}

\subsection{Distributed Communication-Efficient Primal-Dual  Algorithm}
Distributed optimization methods are nontrivial when we train a data mining  problem with dataset partitioned over multiple machines. We suppose that the dataset of $n$ samples is evenly partitioned across $K$ workers. $P_k$ represents the subset of data in the worker $k \in [K] $, where  $|P_k|=n_k$ and $n = \sum_{k=1}^K n_k$. Sample $i \in P_k$ is only stored in the worker $k$, such that it cannot be sampled by  any other workers. $d$ represents the dimensionality of the dataset. In \cite{jaggi2014communication,ma2015adding}, the authors proposed communication-efficient distributed dual coordinate ascent algorithm (CoCoA) for distributed optimization of dual problem. Communication-efficient means that CoCoA allows for more computation in the worker side before communication between workers. Suppose the dataset is partitioned over $K$ machines, and all machines are doing computation simultaneously. In each iteration, workers optimize their  local subproblems $ \mathcal{G}_k^{\sigma'}(\Delta \alpha^{t}_{[k]} ; w^{t}, \alpha^t_{[k]}) $ independently as follows:

{\small
\begin{eqnarray}
\label{cocoa_sub}
&&\max\limits_{\Delta \alpha^{t}_{[k]} \in \mathbb{R}^{n}} \mathcal{G}_k^{\sigma'}\left(\Delta \alpha^{t}_{[k]} ; w^{t}, \alpha^t_{[k]}\right) \nonumber \\
& = & \max\limits_{\Delta \alpha^{t}_{[k]} \in \mathbb{R}^{n}} -\frac{1}{n} \sum\limits_{i \in P_k}  \phi_i^*\left(-\left(\alpha^t + \Delta \alpha^{t}_{[k]}\right)_i \right) - \frac{1}{K}\frac{\lambda}{2} \left\|w^{t}\right\|^2 \nonumber \\
&& -  \frac{1}{ n} \left(w^{t}\right)^T A_{[k]}  \Delta \alpha^{t}_{[k]}
- \frac{\lambda}{2} \sigma' \left\|\frac{1}{\lambda n} A_{[k]}  \Delta \alpha^{t}_{[k]} \right\|^2
\end{eqnarray}}
where $A_{[k]}$ is the data partition on worker $k$ and $\sigma'$ represents the difficulty of the given data partition. It was proved in Lemma 3 \cite{ma2015adding} that sum of local subproblems (\ref{cocoa_sub}) in $K$ workers closely approximate the global dual problem (\ref{dual}). The global variable $w$ is updated after all workers have obtained a $\Theta$-approximate solution to their local subproblems.  Authors in \cite{ma2015adding} claimed that CoCoA shows significant speedups over previous state-of-the-art methods on large-scale distributed datasets. However, the synchronous communication protocol makes CoCoA vulnerable to slow or dead workers. Suppose the normal workers spend $10$ seconds  completing their computation task while a slow worker needs $60$ seconds. In each iteration,  all normal workers have to wait $50$ seconds for the slow worker, which is a tremendous waste of computation resource.

\section{Straggler-Agnostic and Bandwidth-Efficient Distributed Primal-Dual Algorithm}
\label{algo}
In this section, we propose a novel  Straggler-{A}gnostic and {B}andwidth-Efficient Distributed {P}rimal-{D}ual Algorithm (ACPD)  for the high-dimensional data. 


\subsection{ Straggler-Agnostic Server}
As shown in Figure \ref{algorithm}, previous distributed primal-dual methods need to collect information from all workers before updating, suffering from the straggler problem if there are straggler problems. Running time per iteration is entirely dependent on the slowest workers. We overcome the straggler problem by allowing server to update the model as long as a group of workers has been received. 
For example, in Figure \ref{algorithm}, the server just needs to receive messages from two workers.  
The server keeps a model update variable $\Delta \tilde{w}_k$ for each worker, which stores the update of the server model between two communication iterations of worker $k$. After updating the variables on the server, it sends the model update variable  $\Delta \tilde{w}_k$ to  the corresponding workers for further computation. 

\begin{algorithm}[t]
	\begin{algorithmic}[1]
		\STATE \textbf{Initialize: }  Global model: $\tilde w^{0} \in \mathbb{R}^d  = 0$;\\
		Model update: $\Delta \tilde w_k \in \mathbb{R}^d = 0 $;\\ 
		Condition1: $|\Phi| < B \hspace{0.1cm}  and \hspace{0.1cm} t<T-1$;\\
		Condition2: $|\Phi| < K \hspace{0.1cm}  and \hspace{0.1cm} t==T-1$;
		\FOR {$l = 0, 1,..., L-1$}			
		\STATE Update $w^0 = \tilde{w}^l;$
		\FOR {$t= 0,1,... , T-1$}
		\STATE Empty workers set $\Phi = \emptyset$;
		\WHILE { Condition1 or Condition2}
		\STATE Receive  $ \mathcal{F} (\Delta {w}_k )$ from worker $k$, add $k$ to set $\Phi$;
		\STATE Update $\Delta \tilde w_i = \Delta \tilde w_i + \gamma \mathcal{F} ( \Delta {w}_k )$; 
		\ENDWHILE
		\STATE Update {\small   $w^{t+1} = w^{t} + \gamma \sum_{k\in \Phi} \mathcal{F} ( \Delta w_k )$ };
		\STATE Send  $ \Delta \tilde w_k $ to worker $k$ and set $\Delta \tilde w_k = 0$ if $k \in \Phi$;
		\ENDFOR
		\STATE Update $\tilde{w}^{l+1} = w^{T}$;
		\ENDFOR
	\end{algorithmic}
	\caption{Straggler-Agnostic Server}
	\label{alg_server}
\end{algorithm}
Additionally, we also need to control the gap between workers as update information from slow workers may lead to divergence \cite{zhang2016fixing}. Because of our group-wise communication protocol, the local models on workers are usually of different timestamps. It could severely degrade the performance of the method if local models are too stale. To solve this problem, we make the server to collect information from all workers every $T$ iterations, such that all workers are guaranteed to be received at least once within $T$ iterations. Thus, the maximum time delay between local models is  bounded by $T$. A brief description of the procedures in the server is in Algorithm \ref{alg_server}.

\subsection{Bandwidth-Efficient Worker}
Workers are responsible for most of the complicated computations.  There are $n_k$ data on worker $k$, and it is denoted using $A_{[k]}$.  Because of our group-wise communication, we assume the probability of worker $k$ to be received by the server is $q_k$. In each iteration, worker $k$ solves the local subproblem and obtains an approximate solution $\Delta \alpha_{[k]}$, and then it sends the filtered variable $\mathcal{F}(\Delta w_k)$ to the server. Finally, workers receive the global model from the server for further computation.
\begin{algorithm}[t]
	\begin{algorithmic}[1]
		\STATE \textbf{Initialize: } Local model: $w_k \in  \mathbb{R}^d = 0$; \\
		Model update: $\Delta w_k  \in  \mathbb{R}^d = 0$;\\
		Local dual variable: $\alpha_{[k]} \in  \mathbb{R}^{n} = 0$;\\
		\REPEAT
		\STATE $\Delta \alpha_{[k]} = 0$;
		\STATE  Solve subproblem for $H$ iterations and output $\Delta \alpha_{[k]} $:\\
		\hspace{ 0.5cm} $\mathcal{G}_k^{\sigma'}(\Delta \alpha_{[k]} ; w_k + \gamma  \Delta w_k, \alpha_{[k]}) ;$
		\STATE  Update $\alpha_{[k]} =  \alpha_{[k]} + \gamma \Delta \alpha_{[k]} $;
		\STATE Update $\Delta w_k = \Delta w_k +  \frac{1}{\lambda n} A_{[k]} \Delta \alpha_{[k]}$;
		\STATE  {\small Find the $\rho d_{th}$ largest values  in $|\Delta w_k |$ as $c_k$};
		\STATE  Update mask $\mathcal{M}_k = |\Delta w_k| \geq c_k $;
		\STATE Send $ \mathcal{F}(\Delta w_k) =  \Delta w_k \circ \mathcal{M}_k $  to server;
		\STATE Compute  {\small $ \Delta \hat \alpha_{[k]} =
			\lambda n A_{[k]}^{-1} \left( \Delta w_k \circ \neg \mathcal{M}_k\right) $};
		\STATE  Update $\alpha_{[k]} = \alpha_{[k]} - \gamma \Delta \hat \alpha_{[k]}  $; 
		\STATE  Update $\Delta w_k = 0$;
		\STATE  Receive  $\Delta \tilde w_k$ from the server;
		\STATE  Update $ w_k = w_k + \Delta \tilde{w}_k $;
		\UNTIL {convergence}
	\end{algorithmic}
	\caption{Bandwidth-Efficient Worker $k$}
	\label{alg_worker}
\end{algorithm}

\subsubsection{Subproblem  in  Worker $k$}
At first,  worker $k$ finds an approximate solution $\Delta \alpha_{[k]}$ to the local subproblem as follows:
\begin{eqnarray}
\max\limits_{\Delta \alpha_{[k]} \in \mathbb{R}^{n}} \mathcal{G}_k^{\sigma'}\left(\Delta \alpha_{[k]} ; w_k, \alpha_{[k]}\right),
\label{sub}
\end{eqnarray}
where $\sigma' := \gamma B$ and $\mathcal{G}_k^{\sigma'}\left(\Delta \alpha_{[k]} ; w_k, \alpha_{[k]}\right)$ is defined as:
{\small
	\begin{eqnarray}\mathcal{G}_k^{\sigma'}\left(\Delta \alpha_{[k]} ; w_k, \alpha_{[k]}\right):=
	\frac{1}{q_iK}\biggl(  -\frac{1}{n} \sum\limits_{i \in P_k}  \phi_i^*\left(-\left(\alpha + \Delta \alpha_{[k]}\right)_i\right)  
	\nonumber \\
	- \frac{1}{K}\frac{\lambda}{2} \left\|w_k\right\|^2 	- \frac{\lambda}{2} \sigma' \left\|\frac{1}{\lambda n} A_{[k]} \Delta \alpha_{[k]}\right\|^2  -  \frac{1}{ n} \left(w_k\right)^T A_{[k]}  \Delta \alpha_{[k]} 
	\biggr).
	\label{subproblem}
	\end{eqnarray}
}
At each iteration, we sample $i$ randomly  from $P_k$ and compute $\Delta \alpha_i$  supposing other variables $\alpha_{j, j \neq i}$ are fixed.  We repeat this procedure for $H$ iterations. In Algorithm \ref{alg_worker}, $H$ represents the number of iterations before communication, controlling the trade-off between computation and communication. 
There are many fast solvers for the dual problem (\ref{sub}), such as Stochastic Dual Coordinate Ascent (SDCA) \cite{shalev2013stochastic} and  Accelerated Proximal Stochastic Dual Coordinate Ascent (Accelerated Prox-SDCA)\cite{shalev2014accelerated}. Sampling techniques can also be used to improve the convergence of the local solver, such as importance sampling \cite{zhang2015stochastic} and adaptive sampling \cite{qu2015stochastic}. In this paper, we only consider SDCA with uniform sampling as the local solver.

\subsubsection{Sparse Communication on Worker $k$}
After getting an approximate solution $\Delta \alpha_{[k]}$ to the subproblem (\ref{sub}) after $H$ iterations, we compute update for primal variables using:
\begin{eqnarray}
\Delta w_k &=& \Delta w_k + \frac{1}{\lambda n} A_{[k]} \Delta \alpha_{[k]}.
\end{eqnarray}
As shown in Figure \ref{algorithm}, previous methods used to send $	\Delta w_k \in \mathbb{R}^d$ directly to the server. 
However, when the dimension of data is  large, sending and receiving variables with a full dimensionality between server and workers are time-consuming.  
On the contrary, ACPD requires workers to input $\Delta w_k$ into the filter at first and sends the filtered variable $\mathcal{F}(\Delta w_k) \in \mathbb{R}^{\rho d}$ to the server. We implement the filter by simply selecting the elements whose absolute values are the top  $\rho d$ largest. $\mathcal{M}_k(i)=1$ as long as $|\Delta w_k(i)| \geq c_k $ and otherwise $\mathcal{M}_k(i)=0$.  In this way, major update information is kept in the filtered variables, and less communication bandwidth is required. We can easily compress a sparse vector by storing locations and values of the elements. For the purpose of theoretical analysis, at lines $10$-$12$, we put the filtered out update information  back to the local dual variables $\alpha_{[k]}$. In the end, the worker $k$ receives model update variable $\Delta \tilde w_k$ from the server. The communication time is also $q_i O(\frac{1}{q_i} \rho d) = O(\rho d )$ on average. A brief summarization of the procedures on workers is in Algorithm \ref{alg_worker}. 

Lines $10$-$12$ in Algorithm \ref{alg_worker} keep the primal-dual relation (\ref{primal_dual}) always satisfied at each iteration, which is nontrivial for theoretical analysis. However,  matrix inversion computation at each iteration is not practical.  In practice, we simply replace lines 10-12 with: $\Delta w_k \leftarrow \Delta w_k \circ \neg \mathcal{M}_k$, where $\circ$ denotes element-wise multiplication. In the experiments, we show that this simplification does not affect the convergence empirically.

\section{Convergence Analysis}
\label{ana_1}
In this section, we analyze the convergence properties for the proposed method and prove that it guarantees linear convergence to the optimal solution under certain conditions. 
Because of the group-wise communication, at iteration $t$,  local variable $w_k$ in the worker $k$ equals to $w^{d_k(t)} $ from the server, where $w^{d_k(t)}$ denotes stale global variable $w$ with time stamp $d_k(t)$. $\alpha^t_{[k]}$ is local dual variable, where $\alpha_i = 0$ if $i \notin P_k$. Because data $i  \in P_k$ can only be sampled in the worker $k$, it is always true that $\alpha_{[k]}^t = \alpha^{d_k(t)}_{[k]}$.  Therefore, the subproblem in the worker $k$ at iteration $t$ can be written as $\mathcal{G}_k^{\sigma'}(\Delta \alpha^{t}_{[k]} ; w^{d_k(t)}, \alpha^t_{[k]})$. 
Firstly, we suppose that following Assumptions are satisfied throughout the paper:
\begin{assumption}
	We assume all convex $\phi_i$  are non-negative and it holds that $\phi_i(0)\leq 1, \forall i \in [n] $. All samples $x_i$ are normalized such that:
$
	\|x_i\|_2^2 \leq 1.
$
\end{assumption}

\begin{assumption}
	Function $\phi_i$ is $\frac{1}{\mu}$-smooth $\forall i \in [n]$,   such that $\forall a, b \in \mathbb{R}^d$: 
$
	\|\nabla \phi_i(a) - \nabla \phi_i(b)\| \leq \frac{1}{\mu} \|a-b\| \hspace{0.5cm}. 
$
\end{assumption}

\begin{assumption}
	Time delay on worker $k$ is upper  bounded  that:
$
	t - d_k(t) \leq \tau, \hspace{0.5cm} \forall t\geq 0.
$
	In our algorithm, we have $\tau \leq T-1$.
\end{assumption}

For any local solver in the workers,  we assume that the result of the subproblem after $H$ iterations is an approximation of the optimal solutions:
\begin{assumption}
	\label{ass_1}
	We define $\Delta \alpha_{[k]}^*$ to be the optimal solution to the local subproblem $\mathcal{G}_k^{\sigma'}(\Delta \alpha^{t}_{[k]} ; w^{d_k(t)}, \alpha^t_{[k]})$ on worker $k$.  For each subproblem, there exists a constant $\Theta \in [0,1)$ such that:
	{\small	\begin{eqnarray}
		\mathbb{E}_k \left[\mathcal{G}_k^{\sigma'} \left(\Delta \alpha_{[k]}^*; w^{d_k(t)}, \alpha^t_{[k]}\right) - \mathcal{G}_k^{\sigma'} \left(\Delta \alpha^{t}_{[k]}; w^{d_k(t)}, \alpha^t_{[k]}\right)   \right]  \nonumber \\
		\leq \Theta \left(\mathcal{G}_k^{\sigma'} \left(\Delta \alpha_{[k]}^*;w^{d_k(t)}, \alpha^t_{[k]}\right)  - \mathcal{G}_k^{\sigma'} \left(0; w^{d_k(t)}, \alpha^t_{[k]}\right)  \right).
		\end{eqnarray}
	}
\end{assumption}
All above assumptions are commonly used  in \cite{shalev2013accelerated,shalev2013stochastic}. Based on these assumptions, we  analyze the convergence rate of the proposed method.  At first, we analyze the optimization path of the proposed method  at each iteration.

\begin{lemma}
	\label{lem2}
	Suppose all assumptions are satisfied,  $\phi_i$ is $\frac{1}{\mu}$-smooth and convex. $w^t$ and $\alpha^t$  are computed according to  Algorithm \ref{alg_server} and \ref{alg_worker}. Let $\gamma \geq 0$ and $s\in [0,1]$. We define
	$-u_i^t \in \partial \phi_i\left((w^t)^Tx_i\right) $ and $R^t = \frac{\sigma'}{2\lambda} \left(\frac{s}{n}\right)^2\sum\limits_{k=1}^K \left\|A_{[k]} (u^t - \alpha^t)_{[k]} \right\|^2 -\frac{\mu}{2} \left(1-s\right)s \left\|u_i^t - \alpha_i^t\right\|^2,$, if $G(\alpha^t) = 	P(w^t) - D(\alpha^t)$,
	Then following inequality holds that:
	\begin{eqnarray}
	\mathbb{E} \biggl[D(\alpha^t) - D(\alpha^{t+1}) \biggr]  	\leq -\frac{B}{K}\gamma s(1-\Theta) G(\alpha^t)  \nonumber\\
	+   \frac{B}{K}\gamma (1-\Theta) R^t  	+  Q_7.
	\label{20003}
	\end{eqnarray}
	where {\tiny $Q_7= \frac{\gamma^2B}{\lambda nK} \left(  s (1-\Theta)   + 1\right) \sum\limits_{j=d_k(t)}^{t-1} \sum\limits_{i=1}^n \left\|u_i^{d(j)} - \alpha_i^{d(j)}\right\|\left\|u_i^t - \alpha_i^t\right\|$}.
\end{lemma}
According Lemma \ref{lem2}, we analyze convergence rate of the proposed method as follows. 
\begin{theorem}
	\label{them1}
	Suppose that all assumptions are satisfied. We set $\alpha^0=0$ and  $\sigma_{max} \geq  \sigma_k =  \max\limits_{\alpha_k \in \mathbb{R}^n}\frac{ \|A_{[k]} \alpha_{[k]} \|^2}{ \| \alpha_{[k]}\|^2 } $. Through Algorithm \ref{alg_server} and \ref{alg_worker},  we can obtain duality sub-optimality:
	\begin{eqnarray}
	\mathbb{E}_k \left[ D\left(\alpha^*\right) - D(\alpha^L)\right] &\leq& \varepsilon_D,
	\end{eqnarray}
	after $L$ outer iterations as long as:
	\begin{eqnarray}
	L & \geq& \frac{K}{B\gamma (1-\Theta)s} \log \biggl( \frac{1}{\varepsilon_D } \biggr),
	\end{eqnarray}
	where $s = \frac{\lambda \mu n - 2\gamma n (T-1) + \sqrt{\Delta}}{2(\sigma' \sigma_{\max} + \lambda \mu n)}$ and $\Delta = (2\gamma n(T-1) - \lambda \mu n)^2 - \frac{8\gamma n(T-1)}{1-\Theta} ({\sigma'\sigma_{\max}} + \lambda \mu n)$.
\end{theorem}

\begin{proof}
	For $\mu>0$, we get upper bound of  $R^t$ such that:
	{\small
		\begin{eqnarray}
		&R^t=&  \frac{\sigma'}{2\lambda} \left(\frac{s}{n}\right)^2\sum\limits_{k=1}^K \|A_{[k]} (u^t - \alpha^t)_{[k]} \|^2 \nonumber \\
		&& - \frac{\mu}{2} (1-s)s\frac{1}{n}\sum\limits_{i=1}^n(u_i^t - \alpha_i^t)^2 \nonumber \\
		&	\leq& \frac{\sigma'}{2\lambda} \left(\frac{s}{n}\right)^2 \sum\limits_{k=1}^K \sigma_k\|(u^t-\alpha^t)_{[k]}\|^2 - \frac{\mu s(1-s)}{2n} \|u^t-\alpha^t\|^2   \nonumber \\
		&	\leq  &\biggl( \frac{\sigma'}{2\lambda} \left(\frac{s}{n}\right)^2 \sigma
		_{max} -  \frac{\mu s(1-s)}{2n} \biggr)\|u^t-\alpha^t\|^2, 
		\end{eqnarray}
	}
	where $\sigma_{max} \geq  \sigma_k =  \max\limits_{\alpha_k \in \mathbb{R}^n}\frac{ \|A_{[k]} \alpha_{[k]} \|^2}{ \| \alpha_{[k]}\|^2 } $. According to  Lemma \ref{lem2}, we have: 
	{ \small
	\begin{eqnarray}
	\label{iq_2001}
	&& \mathbb{E}_k \left[D(\alpha^t) - D(\alpha^{t+1}) \right] \nonumber \\
	&\leq& -\frac{B}{K}\gamma s(1-\Theta) G(\alpha^t) + \frac{B}{K}\gamma (1-\Theta) R_t +Q_7  \\
	&\leq& -\frac{B}{K}\gamma s(1-\Theta) (D(\alpha^*) - D(\alpha^t))  + \frac{B}{K}\gamma (1-\Theta) R_t +Q_7 \nonumber,
	\end{eqnarray}}
	where the last inequality follows from that $G(\alpha^t) =	P(w^t) - D(\alpha^t) \geq D(\alpha^*) - D(\alpha^t) $. Then we can get the upper bound of $\mathbb{E}_k \left[D(\alpha^*) - D(\alpha^{t+1})\right]$ as follows:
{\small	\begin{eqnarray}
	\label{ineq_dual_op}
	&\mathbb{E}_k \left[D(\alpha^*) - D(\alpha^{t+1})\right] \nonumber \\
	=& \mathbb{E}_k \left[D(\alpha^*) - D(\alpha^{t+1}) + D(\alpha^t) -D(\alpha^t) \right] \nonumber \\
	\leq&  \left(1-\frac{B}{K}\gamma s (1-\Theta) \right) \left(D(\alpha^*) - D(\alpha^{t})\right) +  \frac{B}{K}\gamma (1-\Theta)R_t +Q_7 \nonumber \\
	 =&   \left(1-\frac{B}{K}\gamma s (1-\Theta) \right) \left(D(\alpha^*) - D(\alpha^{t})\right) \nonumber \\
	&+ \left( \frac{ \gamma^2 s (1-\Theta)B}{  \lambda n K }  + \frac{ \gamma^2 B}{  \lambda n K }\right) \sum\limits_{j=d_k(t)}^{t-1} \sum\limits_{i=1}^n \left\|u_i^{d(j)} - \alpha_i^{d(j)}\right\|\left\|u_i^t - \alpha_i^t\right\|\nonumber\\
	& + \frac{B}{K}\gamma (1-\Theta) \biggl( \frac{\sigma'\sigma_{max}}{2\lambda} \left(\frac{s}{n}\right)^2  -  \frac{\mu s(1-s)}{2n} \biggr)\|u^t-\alpha^t\|^2.
	\end{eqnarray} }
	Summing up the above inequality from $t=0$ to $t=T-1$, we have:
	\begin{eqnarray}
	&& \sum\limits_{t=0}^{T-1} \mathbb{E}_k[D(\alpha^*) - D(\alpha^{t+1})]  \nonumber \\
	&\leq &  \sum\limits_{t=0}^{T-1} \left(1-\frac{B}{K}\gamma s (1-\Theta) \right) (D(\alpha^*) - D(\alpha^{t})) \nonumber \\
	&& + \biggl( \frac{B}{K}\gamma (1-\Theta) \biggl( \frac{\sigma'\sigma_{max}}{2\lambda} \left(\frac{s}{n}\right)^2  -  \frac{\mu s(1-s)}{2n} \biggr)  \\
	&& + \left( \frac{ \gamma^2 s (1-\Theta)B}{  \lambda n K }  + \frac{ \gamma^2B }{  \lambda n K }\right)
	\left(T-1\right)  \biggr) \sum\limits_{t-0}^{T-1}\|u^t-\alpha^t\|^2.\nonumber
	\end{eqnarray}
	Therefore, 	as long as, 
	\begin{eqnarray}
	\biggl( \frac{B}{K}\gamma (1-\Theta) \biggl( \frac{\sigma' \sigma_{\max}}{2\lambda} \left(\frac{s}{n}\right)^2  -  \frac{\mu s(1-s)}{2n} \biggr) \nonumber \\
	+ \left( \frac{ \gamma^2 s (1-\Theta)B}{  \lambda n K }  + \frac{ \gamma^2B }{  \lambda n K }\right)
	\left(T-1\right)  \biggr) \leq 0,
	\label{aaaa}
	\end{eqnarray}
	the following inequality holds that:
	{\small
	\begin{eqnarray}
	\label{final_them}
	\mathbb{E}[D(\alpha^*) - D(\alpha^{T})] \leq \left(1-\frac{B}{K}\gamma s (1-\Theta) \right)  \left( D(\alpha^*) - D(\alpha^0)\right) .
	\end{eqnarray}}
	As per our algorithm, we have $D(\tilde\alpha^{l+1}) = D(\alpha^{T})$. Applying (\ref{final_them}) recursively from $l=0$ to $l=L-1$, it holds that:
$
	\mathbb{E}[D(\alpha^*) - D(\tilde\alpha^{L})] \leq \left(1-\frac{B}{K}\gamma s (1-\Theta) \right)^L  \left( D(\alpha^*) - D(\tilde\alpha^0)\right). 
$
	Letting $s = \frac{\lambda \mu n - 2\gamma n (T-1) + \sqrt{\Delta}}{2(\sigma' \sigma_{\max} + \lambda \mu n)}$, where $\Delta = (2\gamma n(T-1) - \lambda \mu n)^2 - \frac{8\gamma n(T-1)}{1-\Theta} ({\sigma'\sigma_{\max}} + \lambda \mu n)$. It is easy to note that when $\gamma \rightarrow 0$, $\Delta > 0$ and $0<s<1$. Therefore, there exists $\gamma>0$ that (\ref{aaaa}) always holds.
	Setting $\varepsilon_D^l = D(\alpha^*) - D(\tilde\alpha^l)$, we have:
	\begin{eqnarray}
	\mathbb{E}[\varepsilon_D^L] &\leq & (1-\frac{B}{K}\gamma s(1-\Theta) )^L \varepsilon_D^0 \nonumber \\
	&\leq & e^{-\frac{B}{K}\gamma s(1-\Theta)L},
	\end{eqnarray} 
	where the second inequality follows from  $(1-\frac{B}{K}\gamma s(1-\Theta))^L \leq e^{-\frac{B}{K}\gamma s(1-\Theta)L}$ and $\varepsilon_D^0 \leq 1$ from Assumption 1.
	Therefore, if the sub-optimality of dual problem  has an upper bound $\varepsilon_D$, $L$ must be bounded that:
	\begin{eqnarray}
	L &\geq & \frac{K}{B\gamma (1-\Theta)s} \log \biggl( \frac{1}{\varepsilon_D } \biggr).
	\end{eqnarray}
\end{proof}

\begin{figure*}[h]
	\centering
	\begin{subfigure}[b]{0.24\textwidth}
		\centering
		\includegraphics[width=1.9in]{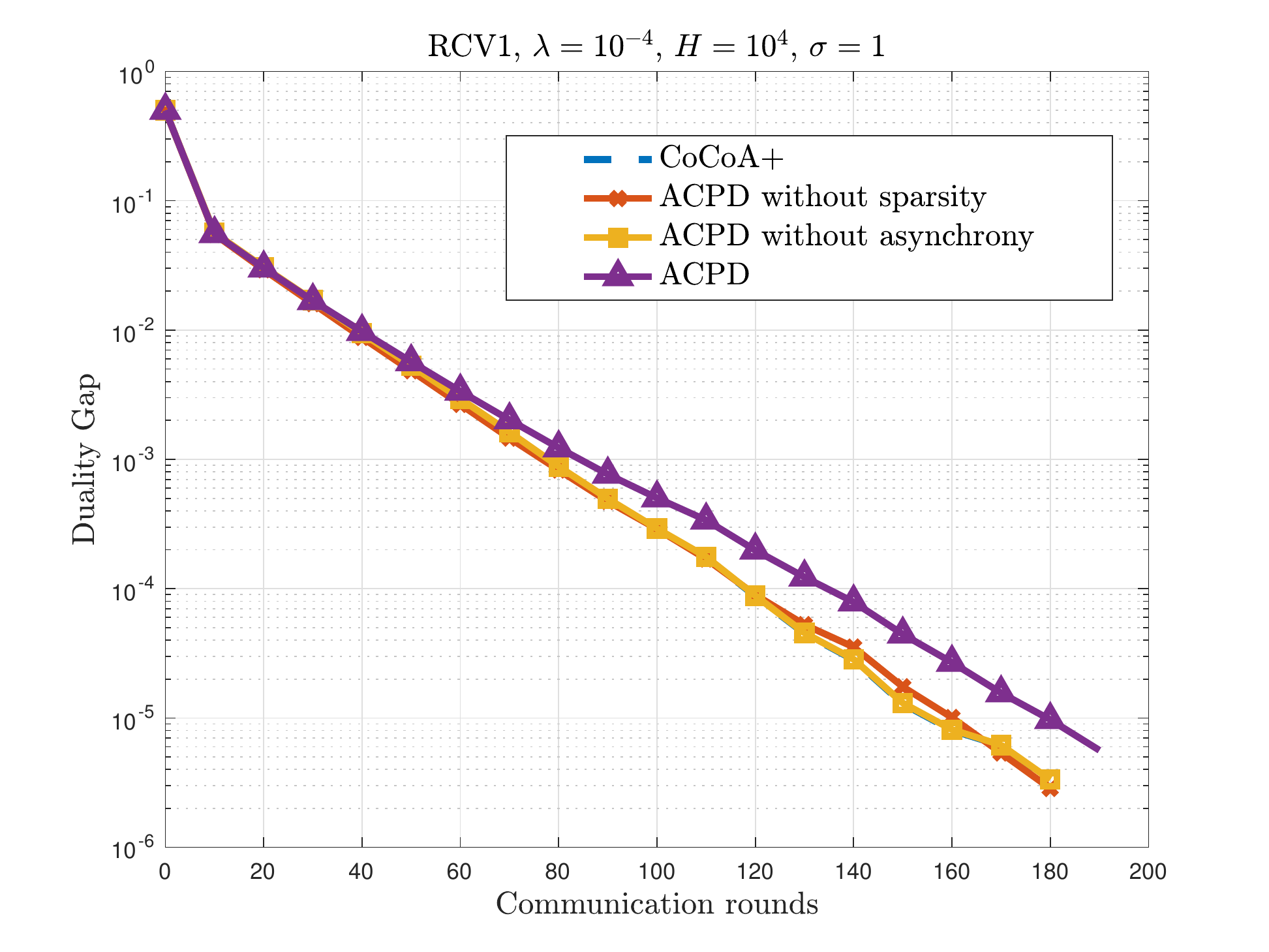}
	\end{subfigure}
	\begin{subfigure}[b]{0.24\textwidth}
		\centering
		\includegraphics[width=1.9in]{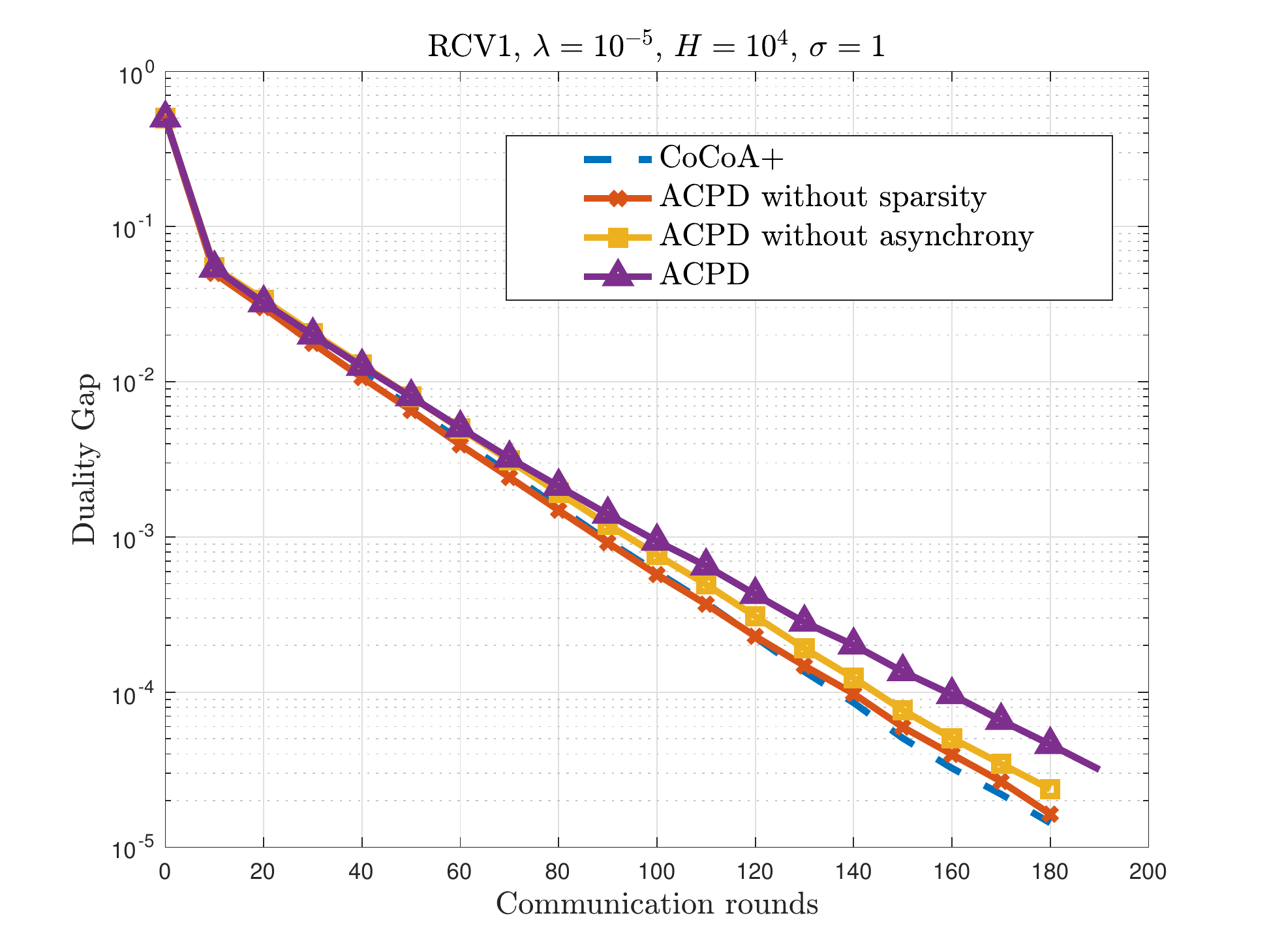}
	\end{subfigure}
	\begin{subfigure}[b]{0.24\textwidth}
		\centering
		\includegraphics[width=1.9in]{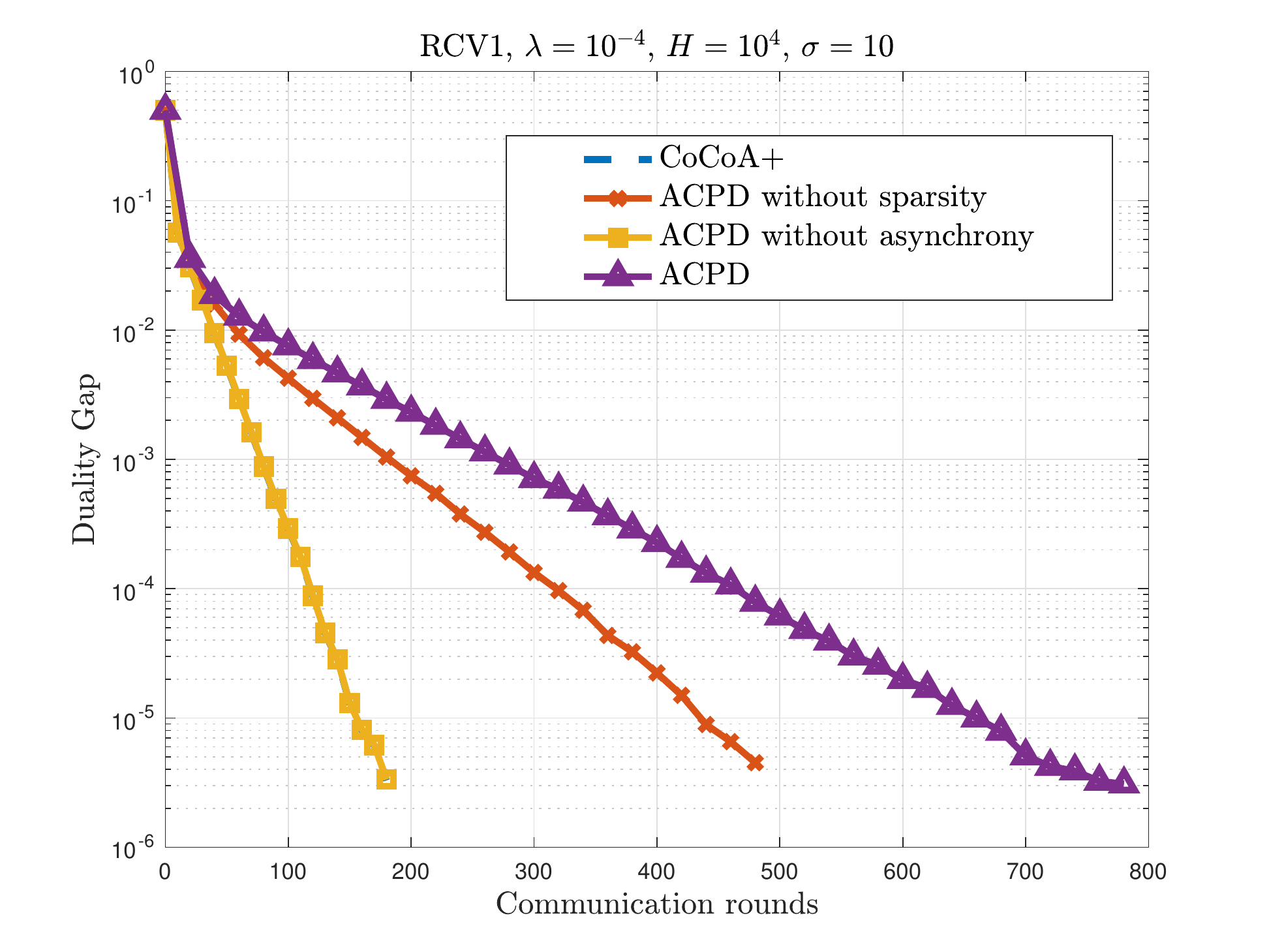}
	\end{subfigure}
	\begin{subfigure}[b]{0.24\textwidth}
		\centering
		\includegraphics[width=1.9in]{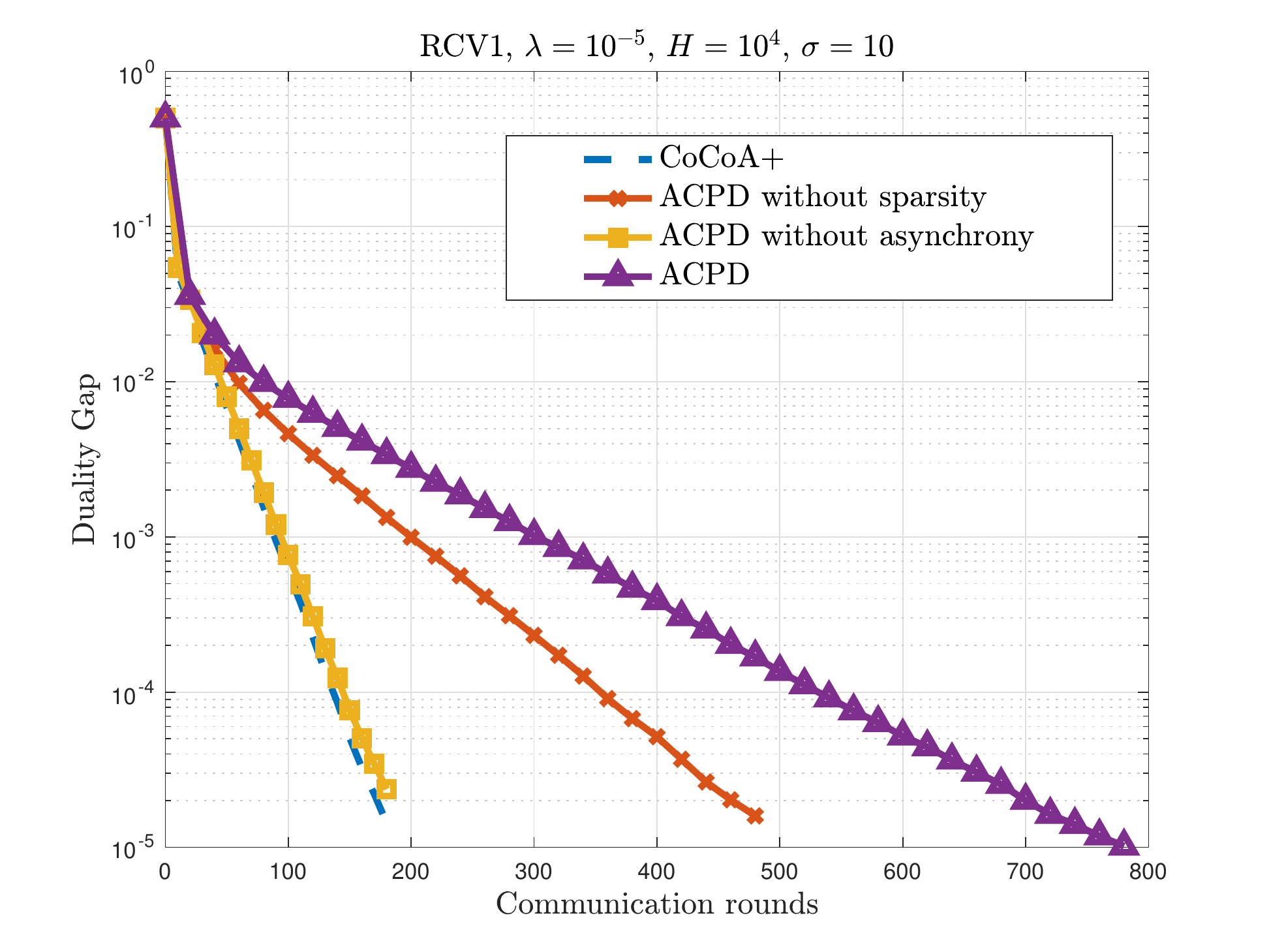}
	\end{subfigure}
	\begin{subfigure}[b]{0.24\textwidth}
		\centering
		\includegraphics[width=1.9in]{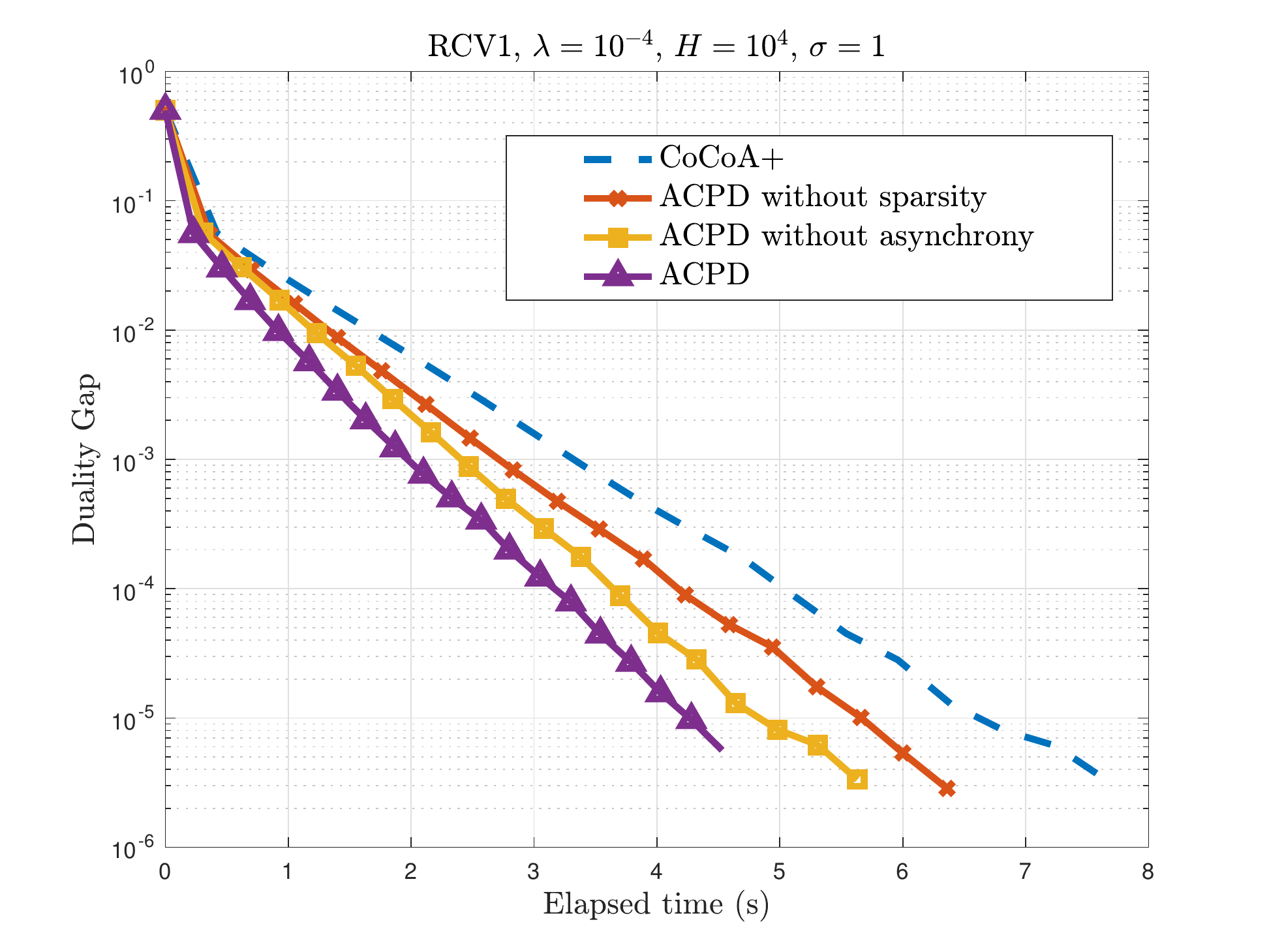}
	\end{subfigure}
	\begin{subfigure}[b]{0.24\textwidth}
		\centering
		\includegraphics[width=1.9in]{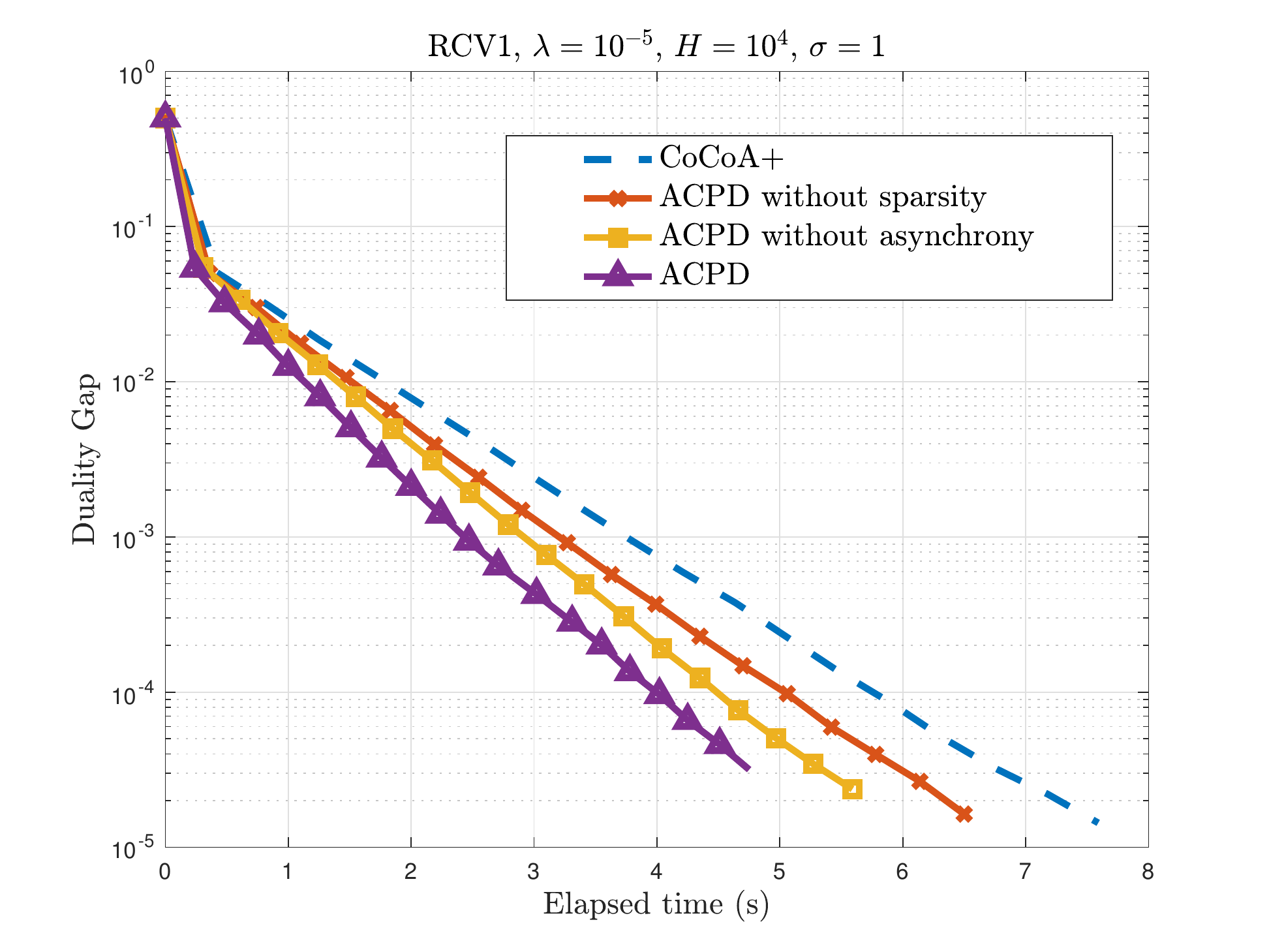}
	\end{subfigure}
	\begin{subfigure}[b]{0.24\textwidth}
		\centering
		\includegraphics[width=1.9in]{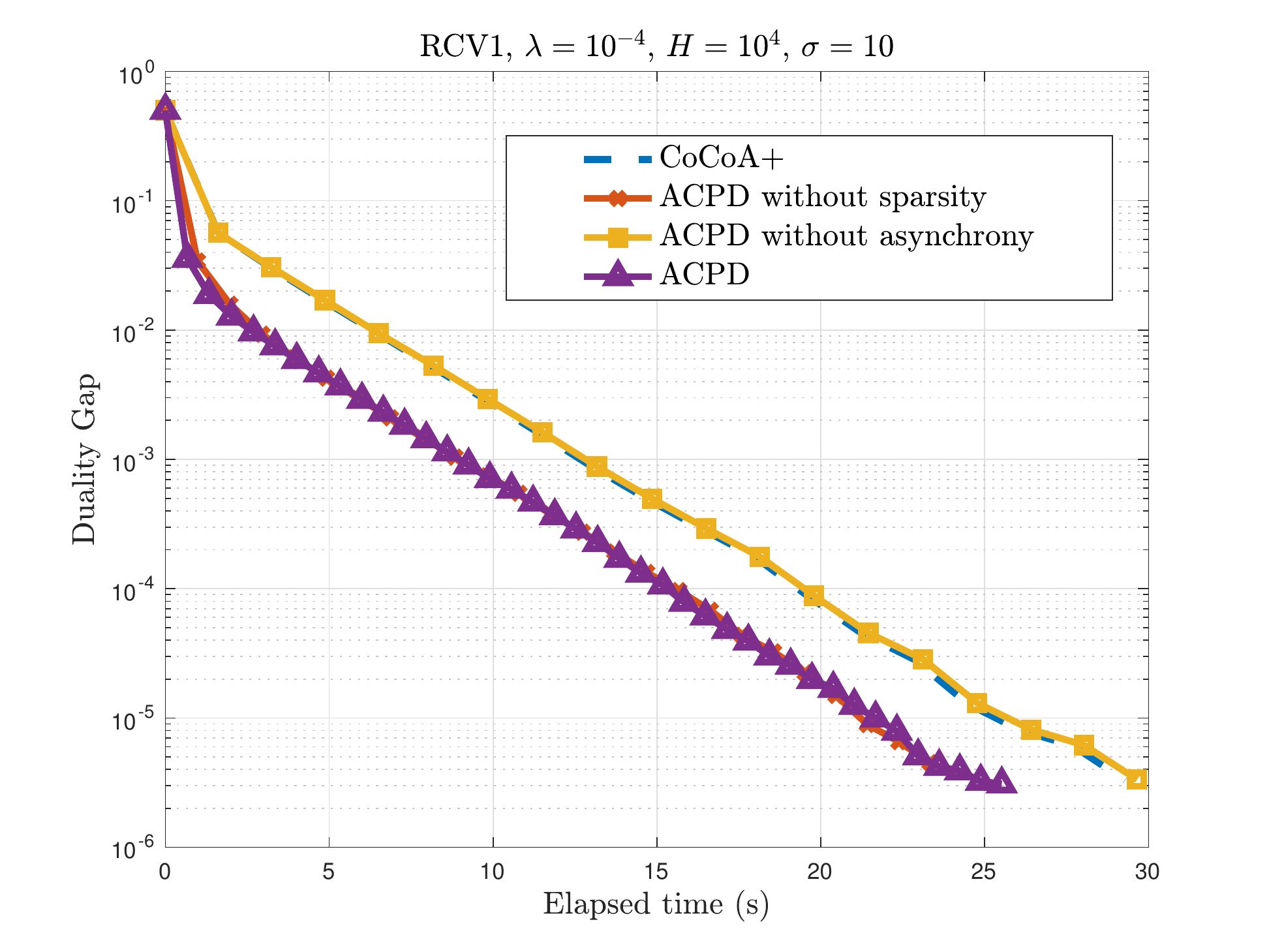}
	\end{subfigure}
	\begin{subfigure}[b]{0.24\textwidth}
		\centering
		\includegraphics[width=1.9in]{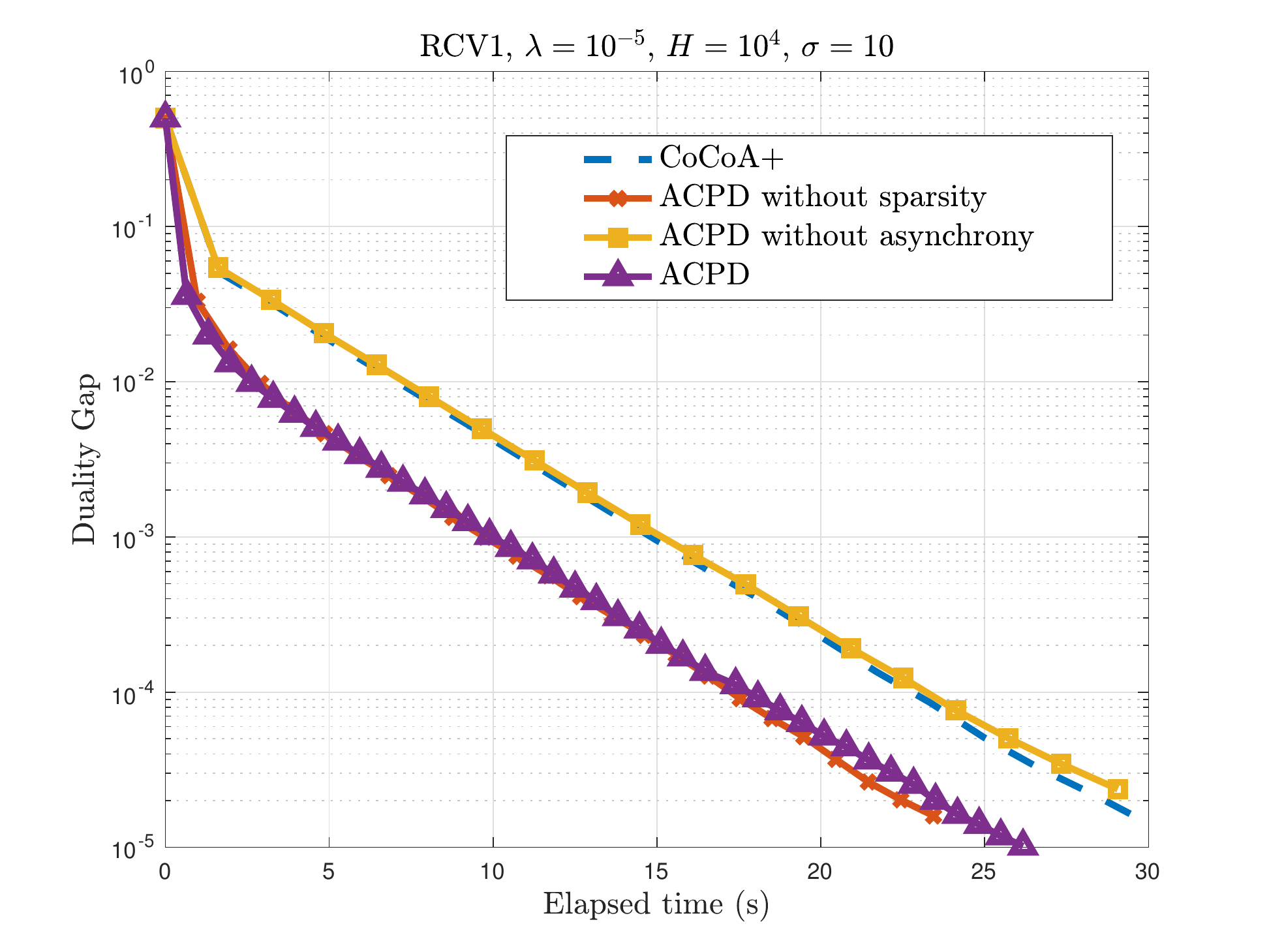}
	\end{subfigure}
	\caption{Convergence of the duality gap for compared methods regarding communication rounds and elapsed time. 
		$\sigma$ denotes the computational time which worker $1$ takes compared to other normal workers. For example, if the normal computational time is $c$, then worker $1$  takes  $\sigma c$.  
	}
	\label{mpi}
\end{figure*}

\begin{remark}
	From Theorem \ref{them1}, when $\gamma \rightarrow 0$, we have $\Delta \rightarrow  (\lambda \mu n)^2 $  and $s \rightarrow \frac{\lambda \mu n}{\lambda  \mu n + \sigma' \sigma_{\max}}$. Therefore, it is guaranteed that we can always find an appropriate $0<\gamma \leq 1$ such that $\Delta > 0$ and $ 0<s<1$ exists.
\end{remark}
\begin{theorem}
	\label{them11}
	Following notations in Theorem \ref{them1}, we can prove that to get duality gap:
$
	\mathbb{E}_k[P(w^L) - D(\alpha^L)] \leq \varepsilon_G,
$
	the outer iteration $L$ must have a lower bound that:
	\begin{eqnarray}
	L &\geq& \frac{K}{B\gamma (1-\Theta)s} \log \biggl( \frac{K}{B\gamma (1-\Theta)s} \frac{1}{\varepsilon_G} \biggr).
	\end{eqnarray}
\end{theorem}

\begin{proof}
	To bound the duality gap, as per (\ref{iq_2001}) we have:
	\begin{eqnarray}
	G(\tilde\alpha^l) &\leq& \frac{K}{B\gamma s (1-\Theta)} (D(\tilde\alpha^l) - D(\tilde\alpha^{l+1}))  \nonumber \\
	&\leq & \frac{K}{B\gamma s (1-\Theta)} (\varepsilon_D^l - \varepsilon_D^{l+1}  )  \nonumber \\
	&\leq & \frac{K}{B\gamma s (1-\Theta)} \varepsilon_D^l .\nonumber \\
	\end{eqnarray}
	Hence, in order to get   $G(\alpha^L) \leq \varepsilon_G$, we must make sure that:
	\begin{eqnarray}
	L &\geq& \frac{K}{B\gamma (1-\Theta)s} \log \biggl(\frac{K}{B\gamma (1-\Theta)s} \frac{1}{\varepsilon_G } \biggr).
	\end{eqnarray}
\end{proof}

Above all, we prove that the proposed method guarantees linear convergence rate for the convex problem as long as Assumptions in this section are satisfied.

\section{Experiments}
\label{exp}
In this section, we perform data mining experiments with large-scale datasets on distributed environments. Firstly, we describe the implementation details of our experiments in Section \ref{exp_sys}. Then, in Section \ref{exp_sim}, we evaluate our method in a simulated distributed environment with straggler problem. Finally, we conduct experiments in a real distributed environment in Section \ref{exp_real}. Experimental results show that the proposed method can be up to 4 times faster than compared methods in real distributed system.

\subsection{Implementations}
\label{exp_sys}
In the experiment, we apply the proposed algorithm to solve a ridge regression problem, where $\phi$  is the least square loss. The dual problem of the ridge regression problem  can be represented as follows :
\begin{eqnarray}
\max\limits_{\alpha \in \mathbb{R}^n} \frac{1}{n} \sum\limits_{i=1}^n \left( \alpha_i y_i - \frac{\alpha_i^2}{2} \right) - \frac{\lambda}{2} \left\| \frac{1}{\lambda n} \sum\limits_{i=1}^n x_i \alpha_i  \right\|^2,
\label{exp_dual}
\end{eqnarray}
where $y_i \in \{-1, 1\}$ and $x_i \in \mathbb{R}^d$ are labels and feature vectors respectively.  We use three binary classification datasets from LIBSVM dataset collections\footnote{ https://www.csie.ntu.edu.tw/~cjlin/libsvmtools/datasets/binary.html}: RCV1, URL and KDD. Table \ref{data} shows brief details about these datasets.  We compare with the CoCoA+ method in the
\begin{table}
	\centering
	\caption{Summary of three real large-scale datasets in the experiment.}
	\begin{tabular}{ c c c  c  }
		\hline
		Dataset &  $\#$Samples (n) & $\#$ Features (d) & Size  \\
		\hline
		RCV1 & $677,399$ & $47,236$ & $1.2$ G \\
		\hline
		URL & $2,396,130$ & $3,231,961$ & $2.1$ G\\
		\hline
		KDD & $19,264,097$ & $29,890,095$ &$4.8$ G \\
		\hline
	\end{tabular}
	\label{data}
\end{table}
experiment, as it has been shown to be superior to other related methods \cite{ma2015adding}.

We implement the compared methods  CoCoA+ and  ACPD using C++, where the communication between workers and server is handled by OpenMPI\footnote{https://www.open-mpi.org/}. We use 'Send' and 'Recv' for point to point communications and 'allreduce' for collective communications.  All experiments are conducted on Amazon Web Services, where each node is a t2.medium instance with two virtual CPUs.

\begin{figure}[!htp]
	\centering
	\begin{subfigure}[b]{0.24\textwidth}
		\centering
		\includegraphics[width=1.9in]{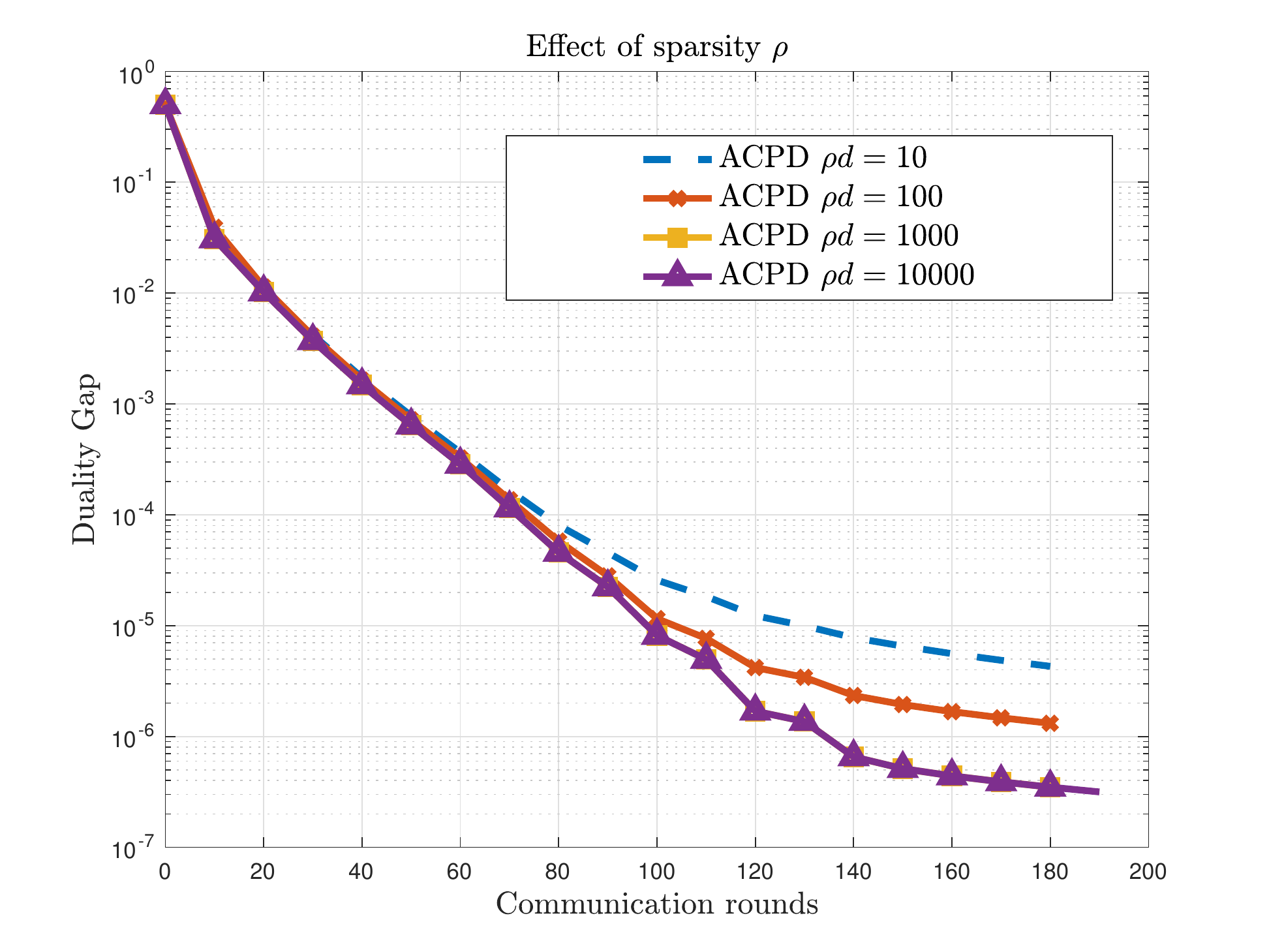}
		\caption{}
		\label{e_sparsity}
	\end{subfigure}
	\begin{subfigure}[b]{0.24\textwidth}
		\centering
		\includegraphics[width=1.9in]{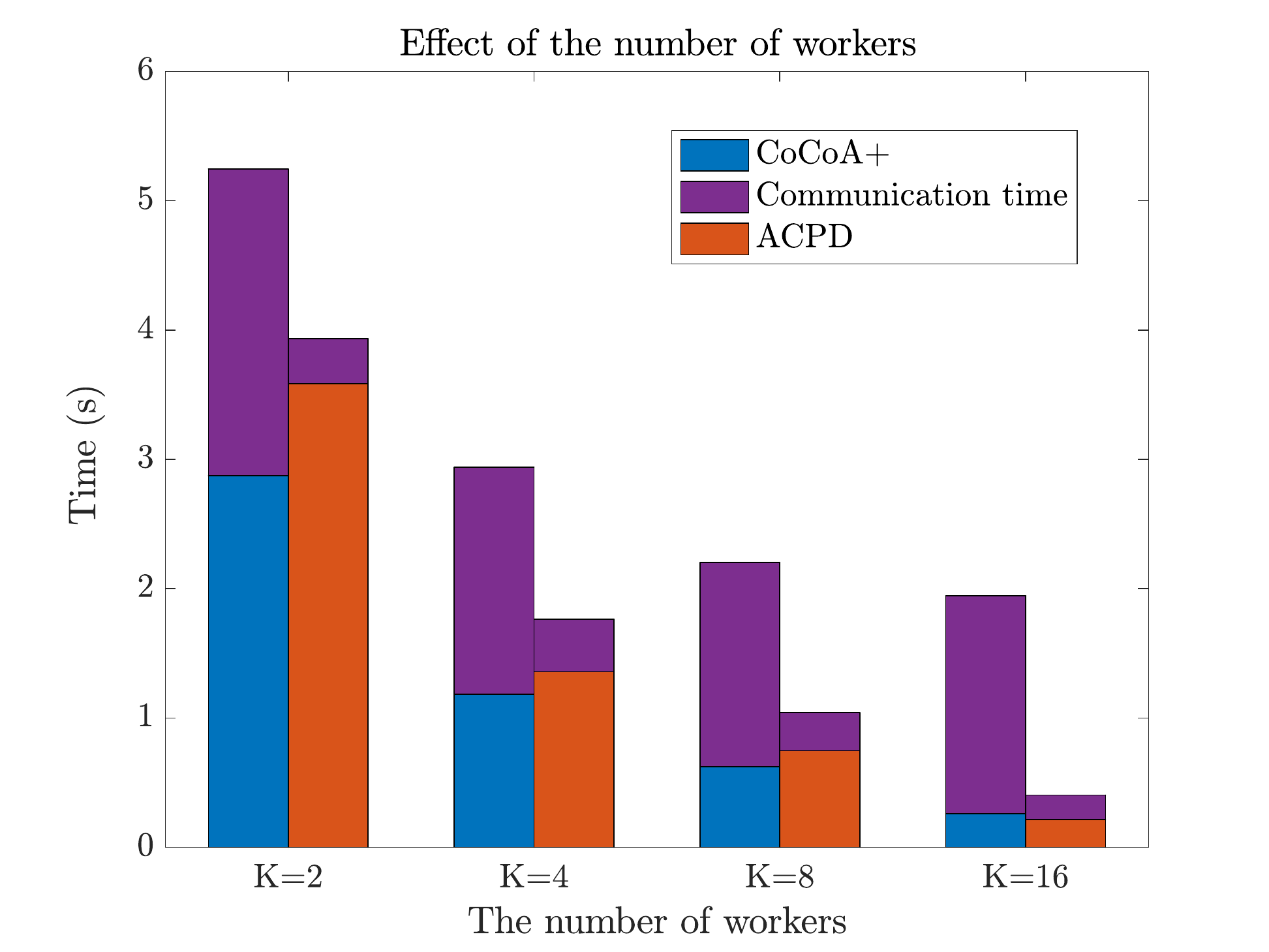}
		\caption{}
		\label{e_t}
	\end{subfigure}
	\caption{Figure \ref{e_sparsity}: Duality gap convergence of the proposed method in terms of the communication rounds with different sparsity constant $\rho$; Figure \ref{e_t}: Total running time of compared methods with a different number of workers.}
\end{figure}

\begin{figure*}[t]
	\centering
	\begin{subfigure}[b]{0.32\textwidth}
		\centering
		\includegraphics[width=2.4in]{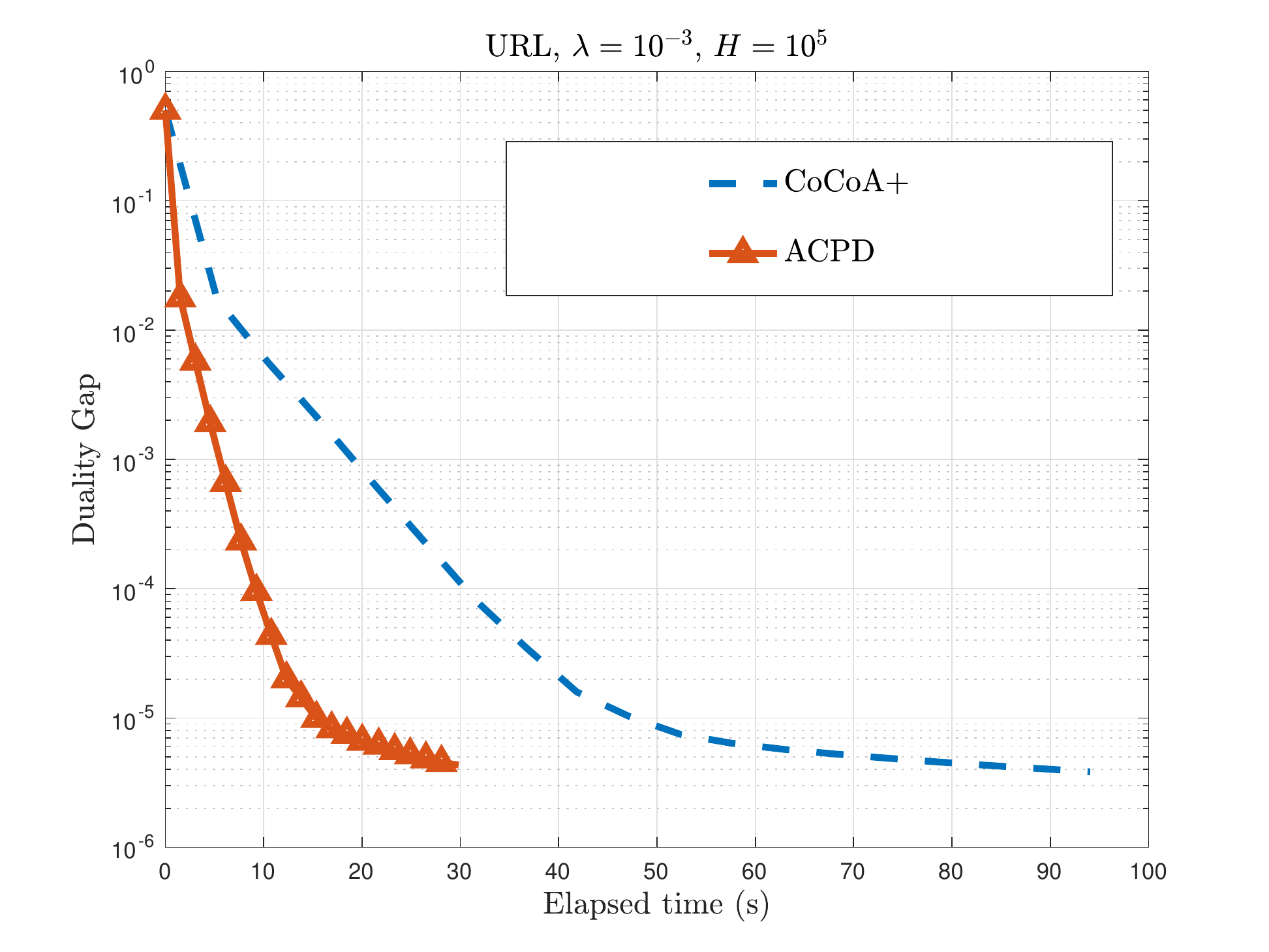}
	\end{subfigure}
	\begin{subfigure}[b]{0.32\textwidth}
		\centering
		\includegraphics[width=2.4in]{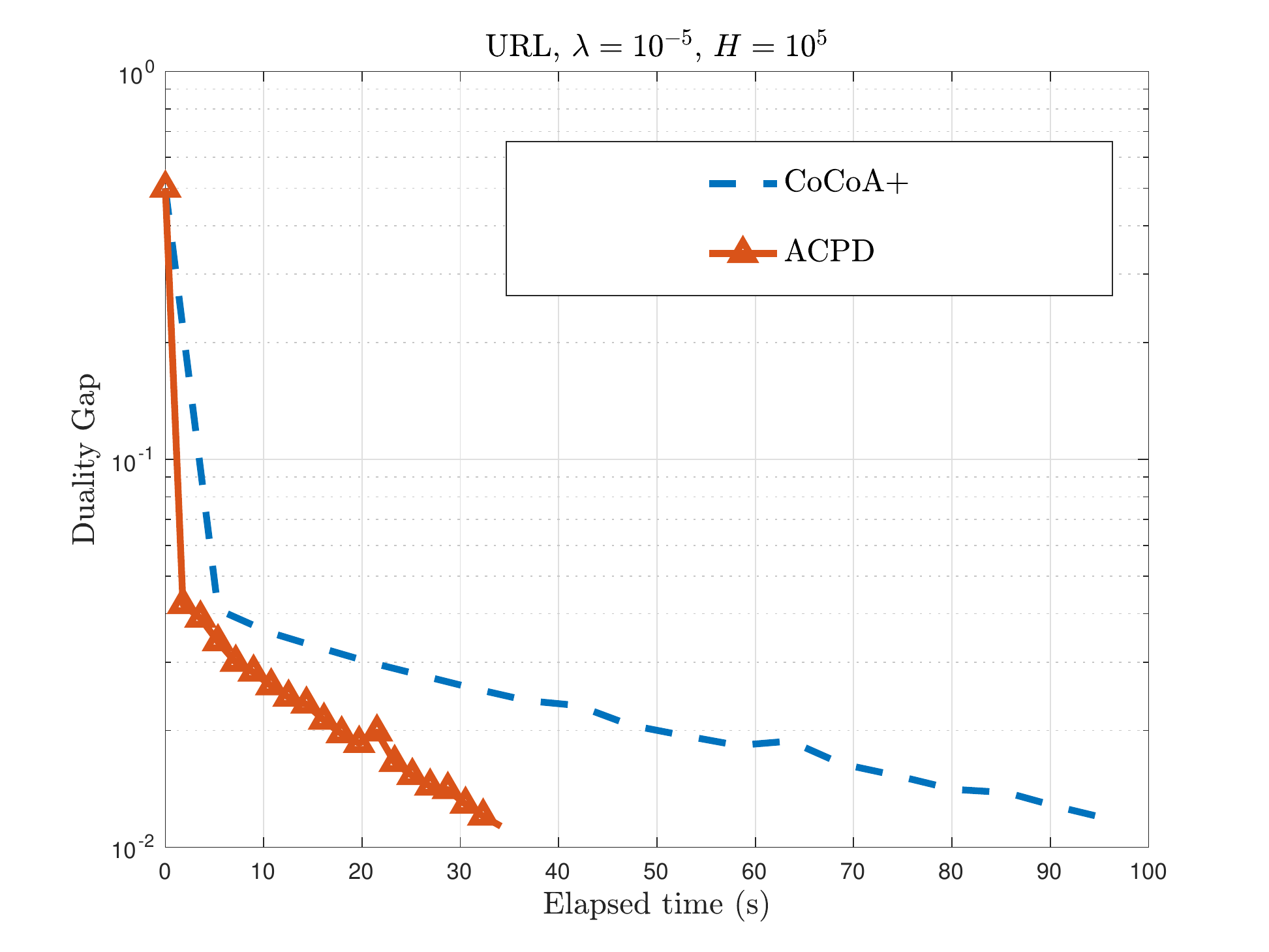}
	\end{subfigure}
	\begin{subfigure}[b]{0.32\textwidth}
		\centering
		\includegraphics[width=2.4in]{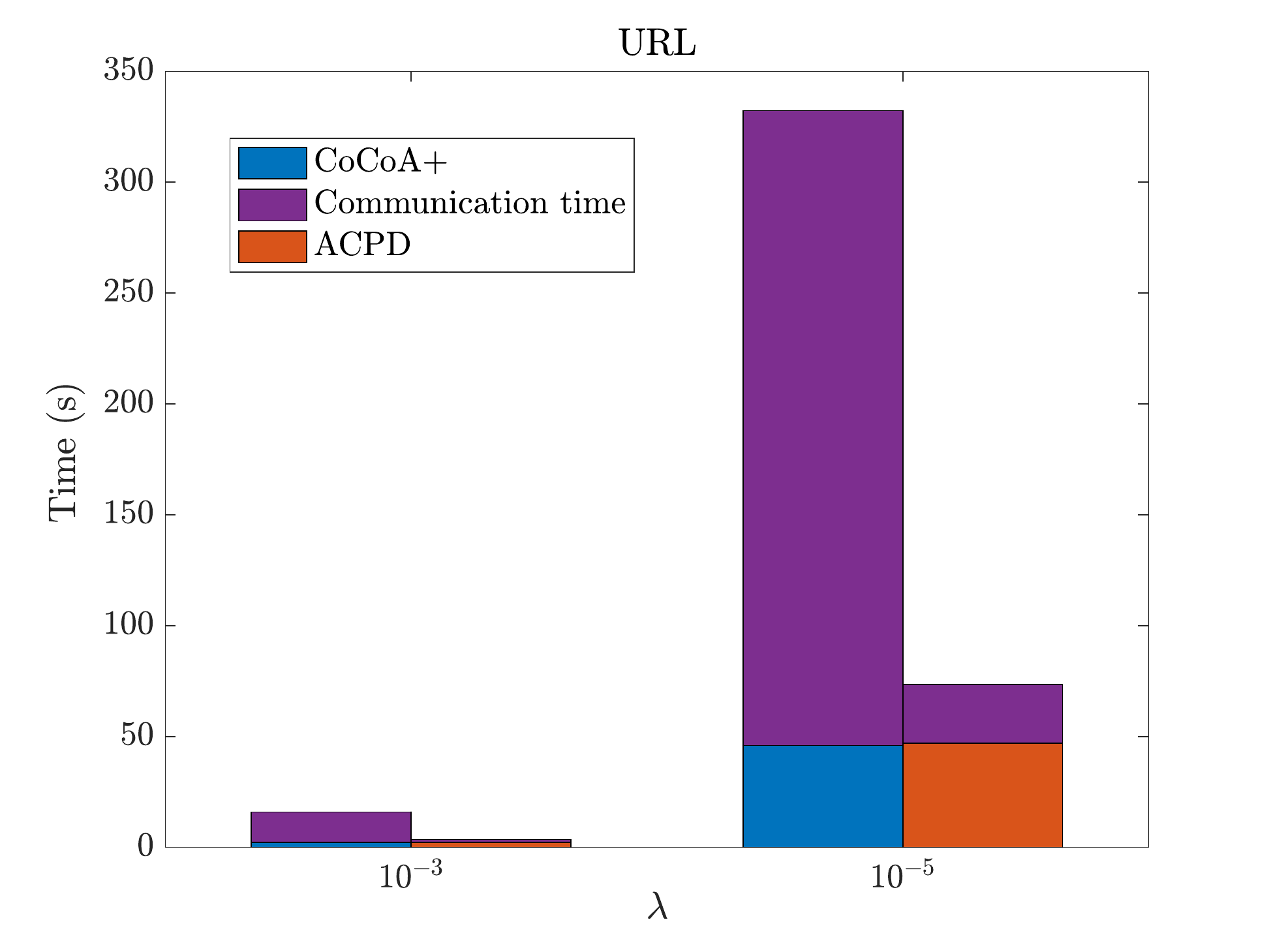}
	\end{subfigure}
	\begin{subfigure}[b]{0.32\textwidth}
		\centering
		\includegraphics[width=2.4in]{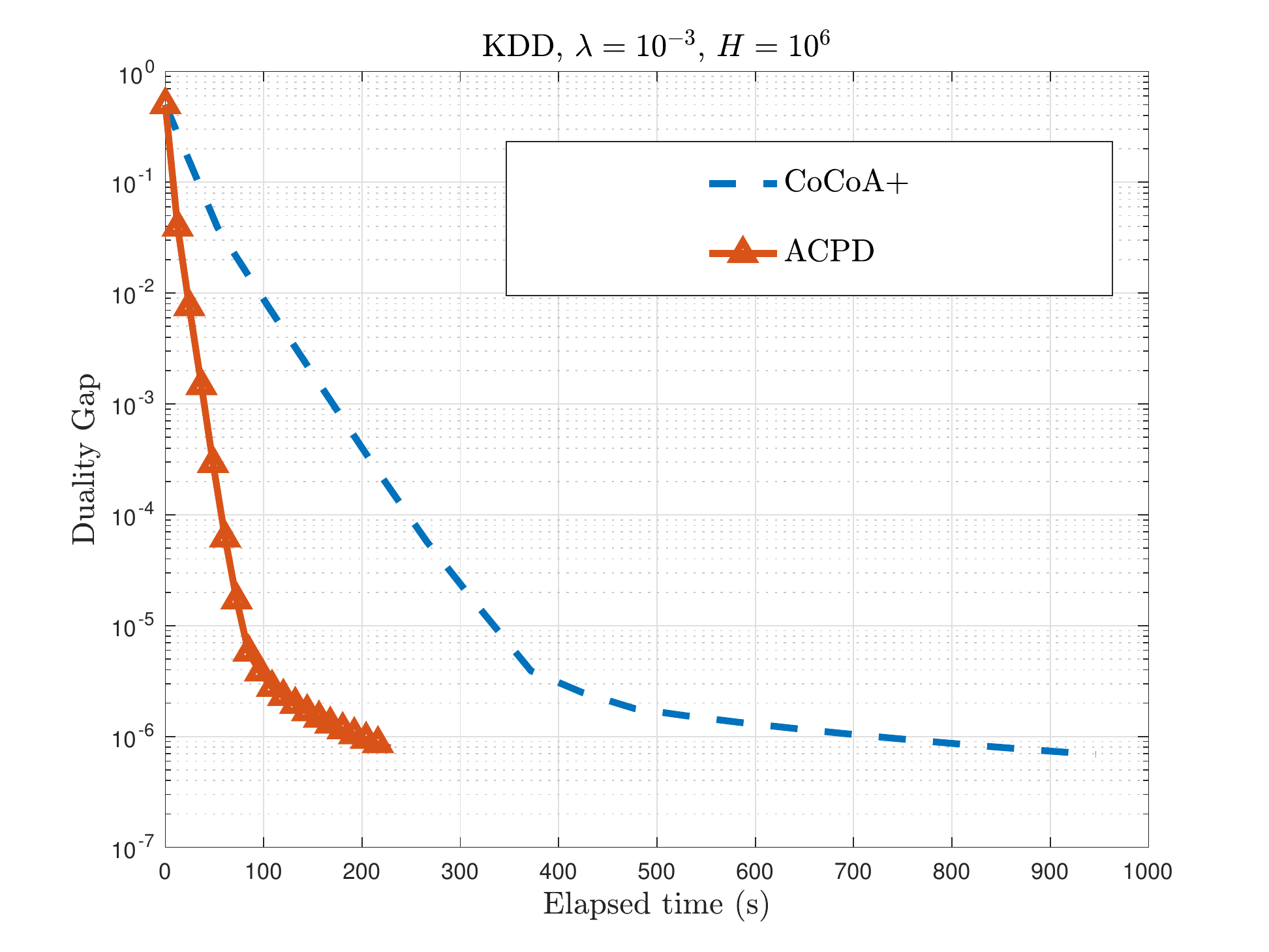}
	\end{subfigure}
	\begin{subfigure}[b]{0.32\textwidth}
		\centering
		\includegraphics[width=2.4in]{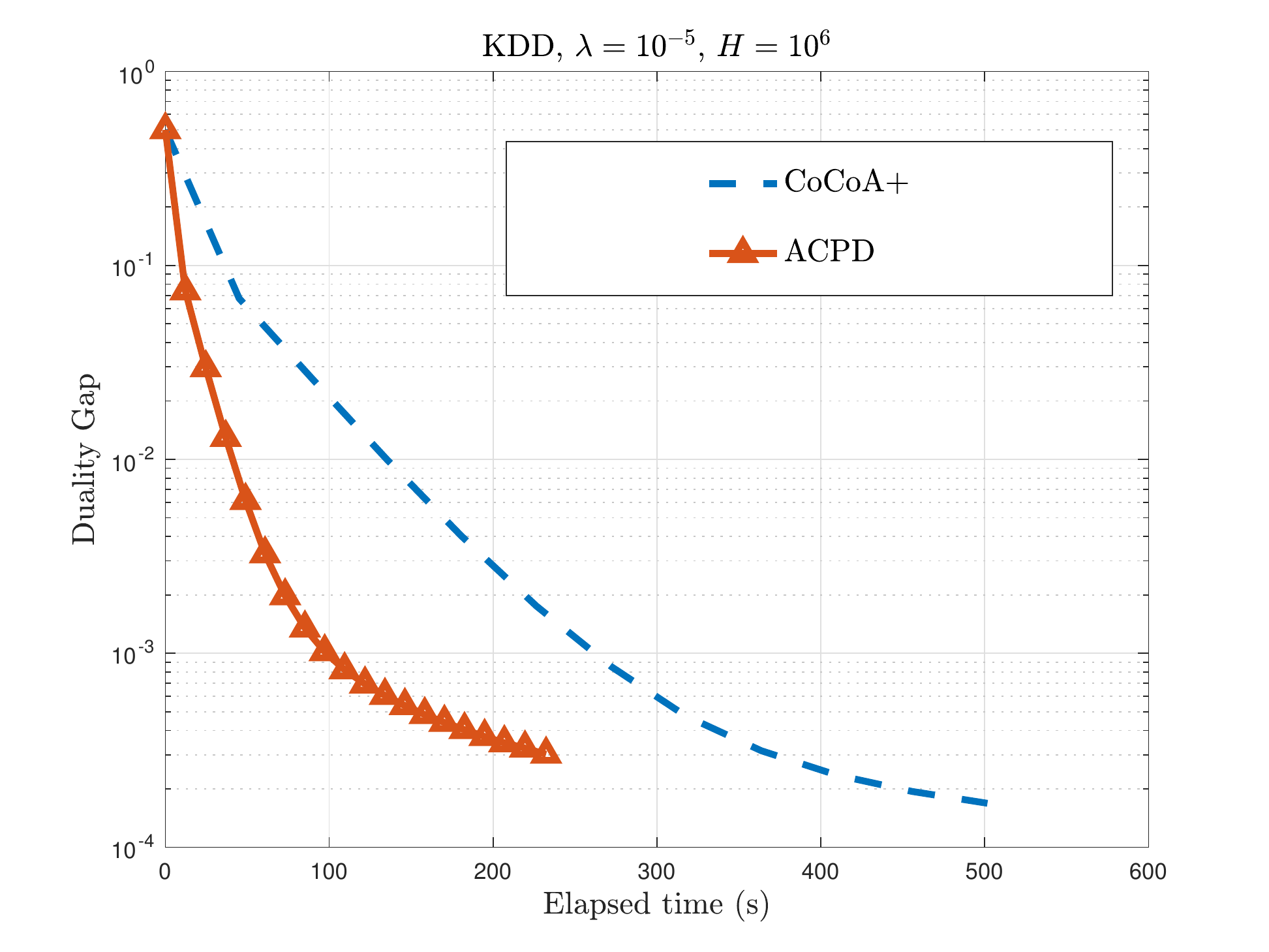}
	\end{subfigure}
	\begin{subfigure}[b]{0.32\textwidth}
		\centering
		\includegraphics[width=2.4in]{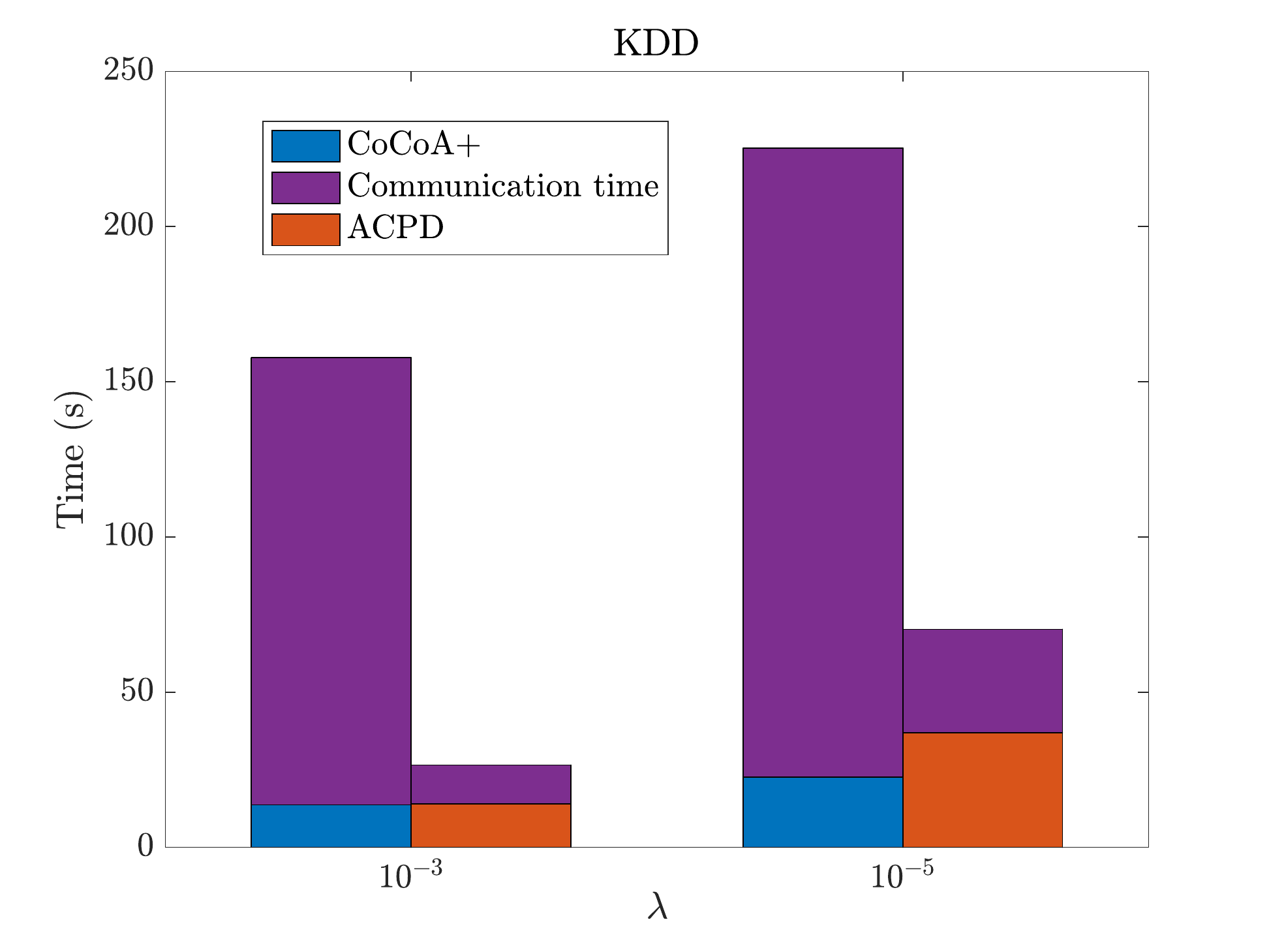}
	\end{subfigure}
	\caption{The left two columns show the convergence of duality gap for the compared methods regarding elapsed time. The right column presents the computational time and communication time of the compared methods when they reach similar duality gap during the optimization.
	}
	\label{real}
\end{figure*}
\subsection{ Experiments in Simulated Distributed Environment}
\label{exp_sim}
In this section, we perform experiments on RCV1 in a distributed system with a  straggler worker. We simulate the straggler problem by forcing worker $1$  to sleep at each iteration such that the computational time of worker $1$ is $\sigma$ times as long as the computational time of other normal workers. 

\subsubsection{Convergence of Compared Methods } 
In Figure \ref{mpi}, we present the convergence rate regarding communication rounds or elapsed time of the compared methods. The RCV1 dataset is distributed across four workers.  We follow the parameter settings in \cite{ma2015adding} for CoCoA+. For ACPD, we set $B=2$, $T=20$ and $\rho d=10^3$ in the experiment. To analyze the effect of straggler-agnosticism and  bandwidth-efficiency, we also do ablation studies by setting $B=K$ or $\rho=1$.  Figures in the first two columns show the results when we set $\sigma=1$. In this situation, the waiting time for the straggler machine is comparable to the communication time between machines. There are two observations: (1) our method admits nearly similar convergence rate to CoCoA+ when $\sigma$ is small, (2)   ACPD converges faster than other compared methods regarding time.  It is obvious that both straggler-agnosticism and bandwidth-efficiency are beneficial for the distributed optimization.  We then set $\sigma =10$ and plot the results in the third and fourth columns of Figure \ref{mpi}. In this situation, the communication time between machines is negligible compared to the waiting time for the straggler machine such that the maximum delay $\tau$ is close to $T$. There are three observations: (1) group-wise communication affects the convergence rate when straggler problem is serious and $\tau$ is large; (2) sparse communication may affect the convergence rate;  (3) when the straggler problem is serious, ACPD is much faster than CoCoA+.

\subsubsection{Effect of Sparsity  Constant $\rho$}
In this section, we evaluate the performance of our methods when we vary the value of the sparsity constant $\rho$. We set $\sigma=1$, $B=2$ and $T=20$ in the experiment, and the RCV1 dataset is distributed across four workers. In Figure \ref{e_sparsity}, we can observe that when we decrease the value of $\rho d$ from $10^4$ to $10$, the convergence rate of ACPD is stable if the magnitude of the duality gap is above $10^{-4}$. The convergence rate of duality gap degrades a little when it is below $10^{-4}$. In practice, we can get good generalization error when the duality gap is above $10^{-4}$. Therefore, our method is  robust to the selection of sparsity constant $\rho$.

\subsubsection{Scaling Up  Workers}
We evaluate the speedup properties of ACPD and CoCoA+ in a distributed system with straggler problem. In the experiment, we set $\sigma =1$, and  $H=10^4$. We test the performance of the compared methods when  $K  \in \{2,4,8,16\} $. For our method, we let $B= \frac{K}{2}$, $\rho d=10^3$ and  $T = 10$. We plot the elapsed time of the compared methods when they reach duality gap $1\times 10^{-4}$  in Figure \ref{e_t}. From this figure, we can observe that our method always spends much less time to reach similar duality gap compared to CoCoA+. As $K$ becomes large, the communication time becomes the bottleneck and constrains CoCoA+ from further speedup. Experimental results show that the group-wise and sparse communication helps ACPD reduce the communication time remarkably, such that it can make better use of resources in a cluster.

\subsection{Experiments in Real Distributed Environment}
\label{exp_real}
In this section, we perform large-scale experiments with KDD and URL datasets in a real distributed environment, where there are other jobs running in different workers. We distribute all datasets on eight workers in AWS platform. For our method, we let $B= 4$, $\rho d=10^3$ and  $T = 10$. Figure \ref{real} presents the performance of the compared methods in terms of computational time. We can observe that our method is much faster than CoCoA+ in the real distributed environment. For example, ACPD is $4$ times faster than CoCoA+ to get duality gap $10^{-6}$ according to the first figure in the second row of Figure \ref{real}.   According to the results in the third column of Figure \ref{real}, the proposed method spends much less communication time than the other compared method.

\section{Conclusion}
In this paper, we propose a novel straggler-agnostic and bandwidth-efficient distributed primal-dual algorithm for high-dimensional data. The proposed method utilizes the group-wise and sparse communication to reduce the communication time in the distributed environment. We  provide the theoretical analysis of the proposed method for convex problem and prove that our method guarantees linear convergence to the optimal solution under certain conditions. Finally, we perform large-scale experiments in distributed environments. Experimental results verify that the proposed method can be up to 4 times faster than compared methods in real distributed system.

\bibliography{asynCoCoA}
\bibliographystyle{plain}

\appendix
\textbf{Appendix} \\
\begin{lemma}
	\label{lem1}
	According to the update rule of $\alpha$ in the Algorithm \ref{alg_server} and \ref{alg_worker}, $\alpha^{t+1} = \alpha^t + \gamma \sum\limits_{k \in \Phi} \Delta \alpha_{[k]}^{d_k(t)} $. If we let $\sigma' = \gamma B$, the following inequality  holds that:
	{\small
		\begin{eqnarray}
		\label{func_lem3}
		\mathbb{E}_k \left[D(\alpha^{t+1}) - D(\alpha^t)\right]  \geq  \frac{B}{K}\gamma \left( \sum\limits_{k=1}^K \mathbb{E}_k\left[ \mathcal{G}_k^{\sigma'} \right]-  D(\alpha^t) \right)  \nonumber \\
		 -  \frac{B}{K}\gamma  \biggl( \frac{\lambda}{2} \|w^t\|^2- \frac{\lambda}{2} \frac{1}{K} \sum\limits_{k=1}^K\left\|w^{d_k(t)}\right\|^2  \nonumber \\
		+ \frac{1}{n} \sum\limits_{k=1}^K\sum\limits_{i\in P_k} (w^t - w^{d_k(t)} )^T x_i \Delta \alpha_i^{t} \biggr).
		\end{eqnarray}
	}
\end{lemma}
\begin{proof}
	There are $K$ workers in the cluster, the possibility of worker $k$ get received is $q_k$, such that $\mathbb{E}_k[\xi_k] = \sum\limits_{k=1}^K q_k \xi_k $ and $\sum\limits_{k=1}^K q_k = 1$.
	From our algorithm, we have $\alpha^{t+1} = \alpha^t + \gamma\sum\limits_{k\in \Phi} \Delta \alpha_{[k]}^{t} $, then it holds that: 
	{\small
		\begin{eqnarray}
		\label{10001}
		&&		\mathbb{E}_k\left[ \frac{1}{q_k K} D\left(\alpha^t + \gamma \sum\limits_{k\in \Phi}\Delta \alpha_{[k]}^{t}\right)\right]\nonumber  \\
		&&= \underbrace{ - \frac{1}{n} \sum\limits_{i=1}^n \mathbb{E}_k\left[ \frac{1}{q_k K} \phi_i^*\left(-\alpha^t_i - \gamma \left( \sum\limits_{k\in \Phi}\Delta\alpha_{[k]}^{t}\right)_i\right)\right]}_{Q_1} \nonumber \\
		&& 	- \frac{\lambda}{2}\underbrace{ \mathbb{E}_k \left[\frac{1}{q_k K} \left\|\frac{1}{\lambda n}A\left(\alpha^t + \gamma\sum\limits_{k\in \Phi} \Delta \alpha_{[k]}^{t} \right) \right\|^2 \right]}_{Q_2}.
		\end{eqnarray}
	}
	Then, we bounds term $Q_1$ and $Q_2$ separately:
	{ \small
		\begin{eqnarray}
		Q_1 
		& = & - \frac{B}{nK} \sum\limits_{k=1}^K \sum\limits_{i \in P_k} \phi_i^*\left(-\alpha^t_i - \gamma \left( \Delta \alpha^{t}_{[k]}\right)_i\right) \nonumber \\
		&& - \frac{K-B}{nK} \sum\limits_{k=1}^K \sum\limits_{i \in P_k} \phi_i^*\left(-\alpha^t_i\right) \nonumber \\
		&\geq & - \frac{B}{nK} \sum\limits_{k=1}^K  \sum\limits_{i \in P_k} \left( \left(1-\gamma \right)\phi_i^*\left(-\alpha^t_i\right) + \gamma \phi_i^*\left(-\left(\alpha^t + \Delta \alpha^{t}_{[k]}\right)_i\right) \right)  \nonumber\\
		&&  - \frac{K-B}{nK} \sum\limits_{k=1}^K \sum\limits_{i \in P_k} \phi_i^*\left(-\alpha^t_i\right)  \nonumber\\
		&=& - \frac{B\gamma}{Kn}\sum\limits_{k=1}^K  \sum\limits_{i \in P_k}  \phi_i^*\left(-\left(\alpha^t + \Delta \alpha^{t}_{[k]}\right)_i\right) \nonumber \\
		&& -\left(1 - \frac{ B }{K}\gamma\right)\frac{1}{n} \sum\limits_{k=1}^K \sum\limits_{i \in P_k} \phi_i^*(-\alpha^t_i) ,
		\end{eqnarray}
	}
	where the second equality follows that there are no duplicates of dataset in different workers, the inequality follows from the convexity of $ \phi_i^*$ and Jensen's inequality.
	{
		\begin{eqnarray}
		Q_2 
		&=& \left\|w^t\right\|^2 + \mathbb{E}_k \left[\frac{1}{q_kK}   \frac{  2\gamma}{\lambda n } \left(w^t\right)^T\sum\limits_{k\in \Phi} A_{[k]} \Delta \alpha^{t}_{[k]}\right] \nonumber \\
		&& +  \left(\frac{\gamma}{\lambda n}\right)^2  \mathbb{E}_k \left[ \frac{1}{q_k K} \left\| \sum\limits_{k\in \Phi}A_{[k]}\Delta \alpha^{t}_{[k]} \right\|^2\right] \nonumber \\
		&\leq &  \|w^t\|^2 +  \frac{2\gamma B }{\lambda nK} \sum\limits_{k=1}^K \left(w^t\right)^T A_{[k]} \Delta \alpha^{t}_{[k]} \nonumber \\
		&& +    \frac{B}{K}\left(\frac{1}{\lambda n}\right)^2\gamma \sigma'   \sum\limits_{k=1}^K\left\|  A_{[k]} \Delta \alpha^{t}_{[k]} \right\|^2 ,
		\end{eqnarray}
	}
	where we define $\sigma' := \gamma B $. Integrating the upper bound of $Q_1$ and $Q_2$ in inequality (\ref{10001}), we have:
	{\small
		\begin{eqnarray}
		&&	\mathbb{E}_k\left[ \frac{1}{q_k K} D\left(\alpha^t + \gamma \sum\limits_{k \in \Phi} \Delta \alpha_{[k]}^{t}\right)\right] \nonumber \\
		&\geq& \left(1-\frac{B}{K}\gamma\right) D(\alpha^t) + \frac{B}{K}\gamma \sum\limits_{k=1}^K \mathbb{E}_k\left[ \mathcal{G}_k^{\sigma'}\left(\Delta\alpha^{t}_{[k]};w^{d_k(t)},\alpha^t_{[k]} \right) \right]  \nonumber \\ \nonumber \\
		&&  -  \frac{B}{K}\gamma  \biggl( \frac{\lambda}{2} \|w^t\|^2- \frac{\lambda}{2} \frac{1}{K} \sum\limits_{k=1}^K\|w^{d_k(t)}\|^2 \nonumber \\
		&& + \frac{1}{n} \sum\limits_{k=1}^K\sum\limits_{i\in P_k} (w^t - w^{d_k(t)} )^T x_i \Delta \alpha_i^{t} \biggr).
		\end{eqnarray}
	}
\end{proof}

\textbf{Proof of Lemma \ref{lem2}}
\begin{proof}
	From our algorithm, primal dual relationship always holds that $w^t = \frac{1}{\lambda n} A\alpha^t$. According to  Lemma 1,
	let $Q_3 =  \frac{B}{K}\gamma  \biggl( \frac{\lambda}{2} \left\|w^t\right\|^2 - \frac{\lambda}{2} \frac{1}{K} \sum\limits_{k=1}^K\left\|w^{d_k(t)}\right\|^2 + \frac{1}{n} \sum\limits_{k=1}^K\sum\limits_{i\in P_k} (w^t - w^{d_k(t)} )^T x_i \Delta \alpha_i^{t} \biggr)$. 
	Define $\Delta \alpha_{[k]}^* =\arg\max\limits_{\Delta \alpha^t_{[k]}} \mathcal{G}_k^{\sigma'} (\Delta \alpha^t_{[k]}; w^{d_k(t)}, \alpha^t_{[k]})$, then we have:
	{\small
		\begin{eqnarray}
		&& \mathbb{E}_k\biggl[\frac{1}{q_k K} \left( D(\alpha^{t}) - D(\alpha^{t+1})\right) \biggr] \nonumber \\
		&\leq&  \frac{B}{K}\gamma   \biggl( D(\alpha^t) -  \mathbb{E}_k\left[\sum\limits_{k=1}^K \mathcal{G}_k^{\sigma'}\left(\Delta\alpha^{*}_{[k]};w^{d_k(t)},\alpha^t_{[k]} \right) \right] \nonumber \\
		&& +  \mathbb{E}_k\left[\sum\limits_{k=1}^K \mathcal{G}_k^{\sigma'}\left(\Delta\alpha^{*}_{[k]};w^{d_k(t)},\alpha^t_{[k]} \right) \right]\nonumber \\
		&& -   \mathbb{E}_k\left[\sum\limits_{k=1}^K \mathcal{G}_k^{\sigma'}\left(\Delta\alpha^{t}_{[k]};w^{d_k(t)},\alpha^t_{[k]} \right) \right]\biggr) + Q_3 \nonumber \\
		&\leq & \frac{B}{K}\gamma   \biggl( D(\alpha^t) -\mathbb{E}_k\left[\sum\limits_{k=1}^K \mathcal{G}_k^{\sigma'}\left(\Delta\alpha^{*}_{[k]};w^{d_k(t)},\alpha^t_{[k]} \right) \right] + Q_3 \nonumber \\
		&& + \Theta \biggl( \mathbb{E}_k\left[\sum\limits_{k=1}^K \mathcal{G}_k^{\sigma'}\left(\Delta\alpha^{t}_{[k]};w^{d_k(t)},\alpha^t_{[k]} \right) \right]  \nonumber \\
		&& -  	\mathbb{E}_k\left[\sum\limits_{k=1}^K \mathcal{G}_k^{\sigma'}\left(0;w^{d_k(t)},\alpha^t_{[k]} \right) \right] \biggr) \biggr)  ,
		\label{in_1001}
		\end{eqnarray}
	}
	where the last inequality follows from Assumption 2. As we know that {\small $\mathbb{E}_k\left[\sum\limits_{k=1}^K \mathcal{G}_k^{\sigma'}\left(0;w^{d_k(t)},\alpha^t_{[k]} \right) \right]= D(\alpha^t) + \underbrace{\frac{\lambda}{2} \left\|w^t\right\|^2 - \frac{\lambda}{2} \frac{1}{K} \sum\limits_{k=1}^K \left\|w^{d_k(t)}\right\|^2}_{Q_4}$}, we have:
{\small	\begin{eqnarray}
	\mathbb{E}_k\biggl[\frac{1}{q_kK} \left( D(\alpha^{t}) - D(\alpha^{t+1}) \right) \biggr]  \leq - \frac{ B}{K}\gamma  \Theta Q_4  + Q_3\nonumber \\
	+  \frac{B}{K} \gamma \left(1-\Theta\right) \underbrace{ \left( D(\alpha^t) - \mathbb{E}_k\left[\sum\limits_{k=1}^K \mathcal{G}_k^{\sigma'}\left(\Delta\alpha^{*}_{[k]};w^{d_k(t)},\alpha^t_{[k]} \right) \right]\right)}_{Q_5}.
	\label{20000}
	\end{eqnarray}}
	We can derive the upper bound of $Q_5$:
	{ \small
		\begin{eqnarray}
		&&Q_5\nonumber\\
		&= &\frac{1}{n} \sum\limits_{i=1}^n \biggl( \phi_i^*(-\alpha_i^t - \Delta \alpha_i^*) - \phi_i^*(-\alpha_i^t) \biggr)  -\frac{\lambda}{2} \left\|w^t\right\|^2 \nonumber \\
		&&+ \frac{\lambda}{2} \frac{1}{K}\sum\limits_{k=1}^K \left\|w^{d_k(t)}\right\|^2  \nonumber \\
		&& + \frac{1}{n} \sum\limits_{k=1}^K \sum\limits_{i\in P_k} (w^{d_k(t)})^T x_i\Delta \alpha_i^* + \sum\limits_{k=1}^K \frac{\lambda}{2} \sigma' \left\|\frac{1}{\lambda n} A_{[k]} \Delta \alpha_{[k]}^* \right\|^2  \nonumber \\
		&\leq & \frac{1}{n} \sum\limits_{i=1}^n \biggl( \phi_i^*\left(-\alpha_i^t - s(u_i^t - \alpha_i^t)\right) - \phi_i^*(-\alpha_i^t) \biggr)  -\frac{\lambda}{2} \left\|w^t\right\|^2 \nonumber \\
		&&  + \frac{\lambda}{2}\frac{1}{K}\sum\limits_{k=1}^K \left\|w^{d_k(t)}\right\|^2 + \frac{1}{n} \sum\limits_{k=1}^K \sum\limits_{i\in P_k} s (w^{d_k(t)})^T x_i (u_i^t - \alpha^t_i)  \nonumber \\ 
		&& +  \sum\limits_{k=1}^K \frac{\lambda}{2} \sigma' \left\|\frac{1}{\lambda n} A_{[k]} s(u^t - \alpha^t)_{[k]} \right\|^2\nonumber \\
		&\leq &  \frac{1}{n}  \sum\limits_{i=1}^n \biggl( s \phi_i^*(-u_i^t) -s \phi_i^*(-\alpha_i^t) -\frac{\mu}{2} (1-s)s(u_i^t - \alpha_i^t)^2  \biggr)      \nonumber \\
		&&  + \frac{\lambda}{2}\frac{1}{K}\sum\limits_{k=1}^K \left\|w^{d_k(t)}\right\|^2  + \frac{1}{n}  \sum\limits_{k=1}^K \sum\limits_{i\in P_k} s (w^{d_k(t)})^T x_i (u_i^t - \alpha^t_i) \nonumber \\
		&& + \sum\limits_{k=1}^K \frac{\lambda}{2} \sigma' \left\|\frac{1}{\lambda n} A_{[k]} s(u - \alpha^t)_{[k]} \right\|^2 -\frac{\lambda}{2} \left\|w^t\right\|^2
		\end{eqnarray}
	}
	where $u_i^t - \alpha_i^t = \Delta \alpha_i^{t}$ and $-u_i^t \in \partial \phi_i((w^{t})^Tx_i)$.  The second inequality follows from the strong convexity of $\phi_i^*$:
	\begin{eqnarray}
	&&	\phi_i^*(-\alpha_i^t - s(u_i^t - \alpha_i^t) ) \nonumber  \\
	& \leq&   s \phi_i^*(-u_i^t) + (1-s) \phi_i^*(-\alpha_i^t) -\frac{\mu}{2} (1-s)s(u_i^t - \alpha_i^t)^2 .
	\end{eqnarray} 
	Because of convex conjugate function maximal property, it holds that:
	\begin{eqnarray}
	\phi_i^*(-u_i) & =& -u_i(w^{t})^Tx_i - \phi_i\left((w^{t})^Tx_i\right).
	\end{eqnarray}
	Additionally, duality gap can be represented as:
$
	G(\alpha^t) =  \frac{1}{n} \sum\limits_{i=1}^n \biggl( \phi_i((w^{t})^Tx_i) + \phi_i^*(-\alpha^t_i) + (w^{t})^Tx_i \alpha^t_i  \biggr).
$
	Therefore,
	{\small
		\begin{eqnarray}
		\label{Q_2_in} 
			Q_5  \nonumber 	
		&\leq& - s \left(P(w^t)-D(\alpha^t)\right) +  \biggl(  \frac{\lambda}{2} \frac{1}{K}\sum\limits_{k=1}^K\left\|w^{d_k(t)}\right\|^2 - \frac{\lambda}{2} \left\|w^t\right\|^2 \nonumber  \\
		&&+ \frac{1}{n} \sum\limits_{k=1}^K \sum\limits_{i\in P_k} s (w^{d_k(t)} - w^t)^T x_i (u_i^t - \alpha^t_i)   \nonumber \\
		&&+  \frac{\sigma'}{2\lambda} \left(\frac{s}{n}\right)^2\sum\limits_{k=1}^K \left\|A_{[k]} \left(u^t - \alpha^t\right)_{[k]} \right\|^2 \nonumber \\
		&&-\frac{\mu}{2} (1-s)s\frac{1}{n}\sum\limits_{i=1}^n(u_i^t - \alpha_i^t)^2   \biggr).
		\label{20001}
		\end{eqnarray}
	}
	Let $Q_6$ represents the right term of the last row in (\ref{Q_2_in}) such that:
	\begin{eqnarray}
	Q_6 &=&  \frac{\lambda}{2} \frac{1}{K}\sum\limits_{k=1}^K\|w^{d_k(t)}\|^2  - \frac{\lambda}{2} \left\|w^t\right\|^2  \nonumber\\
	&& + \frac{1}{n} \sum\limits_{k=1}^K \sum\limits_{i\in P_k} s (w^{d_k(t)} - w^t)^T x_i (u_i^t - \alpha^t_i) \nonumber \\
	& &+  \frac{\sigma'}{2\lambda} \left(\frac{s}{n}\right)^2\sum\limits_{k=1}^K \left\|A_{[k]} (u^t - \alpha^t)_{[k]} \right\|^2 \nonumber \\
	&&-\frac{\mu}{2} (1-s)s\frac{1}{n}\sum\limits_{i=1}^n\left(u_i^t - \alpha_i^t\right)^2.
	\label{20002}
	\end{eqnarray} 
	Input (\ref{20001}) and (\ref{20002}) in (\ref{20000}), we have:
	\begin{eqnarray}
	\mathbb{E}_k\biggl[\frac{1}{q_kK} \left( D(\alpha^{t}) - D(\alpha^{t+1}) \right) \biggr] \leq -\frac{B}{K}\gamma s(1-\Theta) G(\alpha^t) \nonumber \\
	+  {\frac{B}{K}\gamma (1-\Theta)Q_6 -  \frac{B}{K}\gamma  \Theta Q_4 + Q_3}.
	\end{eqnarray}  
	
	Let $R^t = \frac{\sigma'}{2\lambda} \left(\frac{s}{n}\right)^2\sum\limits_{k=1}^K \left\|A_{[k]} (u^t - \alpha^t)_{[k]} \right\|^2 -\frac{\mu}{2} \left(1-s\right)s\frac{1}{n}\sum\limits_{i=1}^n(u_i^t - \alpha_i^t)^2  $, then we have:
{\small	\begin{eqnarray}
	&& {\frac{B}{K}\gamma (1-\Theta)Q_6 -  \frac{ B}{K} \gamma \Theta Q_4 + Q_3} \nonumber \\
	& = & \frac{B}{K}\gamma (1-\Theta) \biggl(  \frac{\lambda}{2}  \frac{1}{K}\sum\limits_{k=1}^K \left\|w^{d_k(t)}\right\|^2 - \frac{\lambda}{2} \left\|w^t\right\|^2 \nonumber \\
	&&+ \frac{1}{n} \sum\limits_{k=1}^K \sum\limits_{i\in P_k} s (w^{d_k(t)} - w^t)^T x_i (u_i^t - \alpha^t_i)  \biggr) \nonumber \\
	&& - \frac{B}{K}\gamma  \Theta  \biggl( \frac{\lambda}{2} \left\|w^t\right\|^2 - \frac{\lambda}{2}\frac{1}{K}\sum\limits_{k=1}^K \left\|w^{d_k(t)}\right\|^2 \biggr)  \nonumber \\
	&&+  \frac{B}{K}\gamma  \biggl(  - \frac{\lambda}{2} \frac{1}{K} \sum\limits_{k=1}^K\left\|w^{d_k(t)}\right\|^2+ \frac{\lambda}{2} \left\|w^t\right\|^2 \nonumber \\
	&& + \frac{1}{n} \sum\limits_{k=1}^K\sum\limits_{i\in P_k} (w^t - w^{d_k(t)} )^T x_i \Delta \alpha_i^{t} \biggr)+  \frac{B}{K}\gamma (1-\Theta) R^t \nonumber \\
	& \leq & \frac{B}{K}\gamma (1-\Theta) R^t  \nonumber \\
	&&+  \frac{B}{K}\gamma s(1-\Theta) \frac{1}{n}  \sum\limits_{k=1}^K \sum\limits_{i\in P_k} \left\|\sum\limits_{j=d_k(t)}^{t-1} \gamma \sum\limits_{k \in \Phi} \frac{1}{\lambda n} A_{[k]}  \Delta \alpha^{j}_{[k]} \right\|\left\|x_i(u_i^t-\alpha_i^t)\right\| \nonumber \\
	&& + \frac{B}{K} \gamma \frac{1}{n}  \sum\limits_{k=1}^K \sum\limits_{i\in P_k} \left\|\sum\limits_{j=d_k(t)}^{t-1} \gamma \frac{1}{\lambda n} A \Delta \alpha^{d(j)} \right\|  \left\|x_i \Delta \alpha^t_i\right\| \nonumber \\
	&\leq & \frac{B}{K}\gamma (1-\Theta) R^t  + {Q_7},
	\end{eqnarray}}
	where the last inequality follows from $\|x_i\| \leq 1$ and we let $Q_7={ \left( \frac{ \gamma^2 s (1-\Theta)B}{  \lambda n K }  + \frac{ \gamma^2 B}{  \lambda n K }\right) \sum\limits_{j=d_k(t)}^{t-1} \sum\limits_{i=1}^n \left\|u_i^{d(j)} - \alpha_i^{d(j)}\right\|\left\|u_i^t - \alpha_i^t\right\| }$ 
\end{proof}

\end{document}